\documentclass{article}
\usepackage{arxiv}
\usepackage{times}
\usepackage{natbib}
\setcitestyle{authoryear,round,citesep={;},aysep={,},yysep={;}}
\usepackage{xcolor}
\usepackage[T1]{fontenc}
\usepackage{amsmath,amssymb}
\usepackage{booktabs}
\usepackage{graphicx}
\usepackage{enumitem}
\usepackage{caption}
\usepackage{amsthm,bm}
\newtheorem{theorem}{Theorem}
\newtheorem{lemma}{Lemma}
\newtheorem{corollary}{Corollary}
\newtheorem{proposition}{Proposition}
\newtheorem{assumption}{Assumption}
\newtheoremstyle{boldremark}
  {3pt}{3pt}{\normalfont}{}{\bfseries}{.}{0.5em}
  {\thmname{#1}\thmnumber{ #2}\thmnote{ {\normalfont(#3)}}}
\theoremstyle{boldremark}
\newtheorem{remark}{Remark}
\newenvironment{restated}[2]{\par\smallskip\noindent\textbf{#1~\ref*{#2}.}\enspace\itshape\ignorespaces}{\par\smallskip}
\usepackage{colortbl}
\definecolor{oursbg}{HTML}{EAF1FB}
\newcommand{\ours}{\cellcolor{oursbg}}
\newsavebox{\privtop}
\newsavebox{\privbot}
\usepackage{multirow}
\usepackage{array}
\usepackage{longtable}
\usepackage{float}
\usepackage{wrapfig}
\usepackage{placeins}
\usepackage[ruled,vlined,longend]{algorithm2e}
\usepackage{titletoc} % appendix table of contents
\SetKwFor{For}{for}{do}{end for}
\SetKwFor{ForEach}{for \normalfont each}{do}{end for}

\definecolor{stdgray}{gray}{0.5}
\newcommand{\std}[1]{{\tiny\textcolor{stdgray}{$\pm$#1}}}
\newcommand{\B}{\mathcal{B}}
\newcommand{\D}{\mathcal{D}}
\newcommand{\Cset}{\mathcal{C}}
\newcommand{\vz}{\bm{z}}
\newcommand{\vm}{\bm{m}}
\newcommand{\vr}{\bm{r}}
\newcommand{\vmu}{\bm{\mu}}
\newcommand{\vu}{\bm{u}}
\newcommand{\vxi}{\bm{\xi}}
\DeclareMathOperator*{\argmin}{arg\,min}

\usepackage{hyperref}
\hypersetup{colorlinks=true,allcolors=blue}

\title{Server-Side Surrogate Replay for
Federated Class-Incremental Learning}

\renewcommand{\undertitle}{}
\date{}

\author{%
\begin{tabular}{c}
\bf Sungmin Kang\textsuperscript{1} \quad
Zhengzhong Tu\textsuperscript{1,$\dagger$} \quad
Sunwoo Lee\textsuperscript{2,$\dagger$} \\
\\[-0.6em]
\normalfont
\textsuperscript{1}Texas A\&M University \quad
\textsuperscript{2}Inha University \\
\normalfont \texttt{sungmin.kang@tamu.edu}
\end{tabular}
}

\begin{document}
\maketitle
\suppressfloats[t]
\thispagestyle{fancy}
{\renewcommand{\thefootnote}{\fnsymbol{footnote}}\footnotetext[2]{Equal advising.}}

\begin{abstract}
Federated class-incremental learning (FCIL) must learn new classes over time
without forgetting the earlier ones.
Existing replay methods counter forgetting on the clients, which store exemplars
or train on synthetic samples.
This burdens resource-constrained clients with extra memory and computation, and
since each client stores exemplars only of the classes it has seen, the replay of earlier
tasks stays as skewed as its local data.
We propose ReSCENE, which moves replay entirely to the server.
Each client uploads a few condensed surrogates per class, and the server
keeps those of earlier tasks from all clients and trains the global model on them
together with the current task, so the replay is balanced across all earlier
classes.
Since the global model and the surrogates co-evolve within a task, the surrogates
of later rounds lie closer to the real data, so we introduce \textit{temporal
herding}, which compresses the pool of each task into a bounded buffer centered
on the recent rounds, and establish a condition under which this buffer
represents the task data more closely than keeping all surrogates.
Across CIFAR-10, CIFAR-100,
and TinyImageNet, ReSCENE attains the highest average accuracy against eight
baselines in all nine settings, by up to $31.1$ points, while requiring as little
as $0.11\times$ of the client computation and up to $179\times$ less upload than
the model-update baselines.
ReSCENE stays ahead with more clients and larger models while keeping client
cost low.
\end{abstract}

\section{Introduction}
\begin{figure*}[t]
\centering
\includegraphics[width=\textwidth]{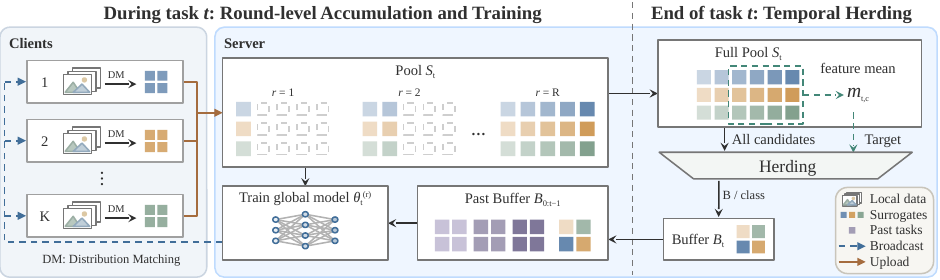}
\caption{Illustration of ReSCENE. During task $t$, clients upload condensed
surrogates every round, and the server trains the global model on the accumulated
pool together with the buffers of earlier tasks. At the end of the task, temporal
herding compresses the pool into a fixed-size buffer centered on the recent rounds,
which joins the past buffer. Darker tiles denote more recent rounds.}
\label{fig:method_overview}
\end{figure*}

Federated learning (FL) has emerged as a privacy-preserving paradigm that trains 
a shared model
across data owners without centralizing their data~\citep{pmlr-v54-mcmahan17a}.
Standard FL assumes that each client's local data is static, yet in practice 
new classes keep arriving and the data distribution shifts over time. 
Thus, the shared model must keep learning from new data while retaining
previously acquired knowledge. This dynamic setting within a federated
environment is \textit{federated continual learning}~\citep{pmlr-v139-yoon21b}. 
We study its class-incremental form, \textit{federated class-incremental
learning} (FCIL)~\citep{Dong_2022_CVPR}, in which each task introduces new
classes and the model is evaluated on all classes seen so far.
Each time the model trains on a new task, earlier knowledge is progressively
overwritten as its parameters shift toward the new task, so the model fails
on the tasks it learned before.
This loss of earlier-task ability is
\textit{catastrophic
forgetting}~\citep{McCloskey1989CatastrophicII,ICLR2024_cf53f12a}, 
and mitigating it is the core problem of FCIL.

A common way to counter forgetting in FCIL is replay, which rehearses past data
together with the new task.
Existing replay methods store exemplars of earlier tasks on the
clients~\citep{Dong_2022_CVPR,Li_2024_CVPR,NEURIPS2025_d611d06e}, or synthesize
past data with a generator at the server and let the clients train on
it~\citep{Zhang_2023_ICCV,NEURIPS2023_d160ea01,lander,10612802}.
In both cases the clients rehearse the replay data during local training, so the
replay loads resource-constrained clients with extra memory and computation.
As tasks accumulate, the clients must store and train on replay data for more
classes, or spread a fixed budget more thinly over the earlier classes.
Moreover, each client stores exemplars only of the earlier classes it has seen, so the replay
of earlier tasks stays as skewed as its local data throughout the task sequence.

Instead of adding the burden to resource-constrained clients, we shift the replay
entirely to the resource-rich server~\citep{10.1561/2200000083}.
In every round, each client condenses its local data into a few surrogates per
class by distribution matching~\citep{Zhao_2023_WACV,Xiong_2023_CVPR} and uploads
them, and the server keeps the surrogates of earlier tasks at no additional
communication cost and trains the global model on them together with those of
the current task (Figure~\ref{fig:method_overview}).
On the clients, local training is replaced by condensation, which takes
less computation and keeps nothing from earlier tasks, so the client cost
stays the same regardless of the number of tasks.
At the server, the replay gathers the surrogates of every class any client
has held and rebalances the classes, so the skew of each client's
local data does not carry over to the replay.
We call this method ReSCENE (\textit{re}\,+\,\textit{scene}), since the server 
re-sees the scene of previous tasks while it learns the current one.

Within a task, the global model and the surrogates co-evolve, since each round's
surrogates train the model used for condensation in the next round.
As the model learns to separate the classes of the task over the rounds, the
surrogates condensed with it move closer to the real data of their clients
under the global model at the end of the task (Figure~\ref{fig:coevolution}).
Based on this observation, we introduce \textit{temporal herding}, which compresses
the pool of each task into a smaller buffer centered on the recent rounds,
selecting from the whole pool by herding~\citep{10.1145/1553374.1553517}.
Our theoretical analysis establishes a condition under which this buffer
represents the task data more closely than the whole pool, and we empirically
show that the buffer can attain higher accuracy than the whole pool, by up to
$6.07$ points on TinyImageNet, while using less server
memory and computation.

Across CIFAR-10, CIFAR-100, and TinyImageNet at three heterogeneity levels,
ReSCENE outperforms the strongest of eight FCIL baselines by up to $31.1$
points of average accuracy, while requiring as little as $0.11\times$ of the client computation of
FedAvg and an upload up to $179\times$ smaller than a model update.
ReSCENE also remains effective with larger client populations and larger models
while keeping this efficiency. We summarize our main contributions as follows:
\begin{itemize}[leftmargin=1.2em,itemsep=2pt,topsep=0pt,partopsep=0pt,parsep=0pt]
\item We propose ReSCENE, which shows that replay can be moved entirely to the
server with condensed surrogates, relieving resource-constrained clients.
\item We find that the global model and the surrogates co-evolve within a task, and
introduce temporal herding with a condition under which its compact buffer
represents the task data more closely than the whole pool.
\item We show that ReSCENE outperforms eight baselines on three FCIL benchmarks at
three heterogeneity levels with less client computation and upload, and that this
holds with more clients and larger models.
\end{itemize}

\section{Related Work}
\label{sec:related}

\noindent\textbf{Replay in FCIL.}\quad
Replay-based methods rehearse earlier-task data alongside the new data.
Exemplar-based replay methods store samples of earlier tasks on the clients and
train on them in local training.
GLFC~\citep{Dong_2022_CVPR} keeps a per-class exemplar memory on each client,
ECoral~\citep{sun2025exemplarcondensedfederatedclassincrementallearning}
stores condensed exemplars, Re-Fed~\citep{Li_2024_CVPR} caches the
previous-task samples that a personalized model scores as most informative, and
FedCBDR~\citep{NEURIPS2025_d611d06e} lets the server choose the exemplars from
obfuscated client features so that the exemplars of all clients together are
class-balanced, while each selected sample stays on its client.
These methods retain earlier knowledge as the clients train
on data of earlier tasks, but the clients keep the exemplars and train on them
throughout the task sequence, which consumes the scarce memory and computation of
resource-constrained clients.
Moreover, each client replays only the exemplars it holds, so its replay remains
skewed toward its own classes.

Generative replay methods train a generator that imitates the data of earlier
tasks and let the clients train on its outputs.
TARGET~\citep{Zhang_2023_ICCV} and MFCL~\citep{NEURIPS2023_d160ea01} train a
generator at the server by inverting the global model and send its samples or the
generator itself to the clients.
LANDER~\citep{lander} conditions this generator on label-text embeddings
of a pre-trained text encoder, and GenFCIL~\citep{10612802} generates features of
the earlier classes in place of images.
FedCIL~\citep{qi2023better} instead trains a generator on every client and
fine-tunes the averaged global model at the server on their samples.
Training the generator often relies on the batch normalization statistics of the global model or on a
pre-trained text encoder.
In addition, the clients still train on the generated samples,
which adds computation to every client.
We discuss these methods and the regularization- and distillation-based families
in detail in
Appendix~\ref{app:related}.

\noindent\textbf{Surrogate-aggregation FL.}\quad
Surrogate-aggregation FL has been introduced as an alternative to
traditional model-transmitting FL. 
In FedDM~\citep{Xiong_2023_CVPR}, each client condenses its local data into a
small set of surrogates, on which the server trains the global model.
FedAF~\citep{Wang_2024_CVPR} follows the same structure and lets each client synthesize surrogates that reflect the global data distribution, which tackles data heterogeneity. 
Both study static training on fixed local data, so the server uses each
surrogate once and then discards it.
ReSCENE keeps the uploaded surrogates at the server at no extra
communication cost and replays them as later tasks arrive. 
\section{Method}
\label{sec:method}

We propose ReSCENE, a server-side surrogate replay method for FCIL. In every round,
participating clients condense their local data into surrogates with the global
model and upload them, and the server trains the global model on these surrogates
together with the buffers of earlier tasks (Sec.~\ref{sec:safl}). At the end of
each task, the server compresses the surrogates into a buffer by temporal herding
(Sec.~\ref{sec:temporal_herding}), which we analyze in Sec.~\ref{sec:theory_short}.
Figure~\ref{fig:method_overview} and Algorithm~\ref{alg:main} summarize the
procedure.

\subsection{Problem Setup and Notation}
\label{sec:setup_method}

We consider FCIL with $N$ clients and $T$ tasks, each of $R$ rounds. Task $t$
introduces a disjoint set of new classes $\Cset_t$, from which client $k$ draws
its local data $\D_t^k$.
In each round $r$ of task $t$, a subset $\mathcal K_{t,r}$ of $K$ clients
participates.
We denote the global model by $f_\theta$ and its feature extractor by $\phi_\theta$,
with parameters $\theta_t^{(r)}$ after round $r$ of task $t$.
On the server, $S_t$ is the pool of surrogates collected during task $t$, 
$\B_t$ is the replay buffer kept from task $t$, and $\B_{0:t}=\bigcup_{t'\le t}\B_{t'}$ 
denotes the buffers of tasks $0$ through $t$.

\subsection{Server-Side Surrogate Replay}
\label{sec:safl}

\noindent\textbf{Surrogate aggregation.}\quad
Unlike conventional model-aggregation FL, clients in ReSCENE upload condensed
surrogates of their local data in place of model updates~\citep{Xiong_2023_CVPR,Wang_2024_CVPR}. In each round $r$ of task $t$,
client $k$ condenses its local data by distribution
matching~\citep{Zhao_2023_WACV}. For each class $c$, it optimizes a small
synthetic set $S_{t,c}^{k,(r)}$ so that its mean embedding and mean logits match
those of the client's real data of that class,
$\D_{t,c}^k$, under the global model $\theta_t^{(r-1)}$ received from the server,
which the client keeps fixed during condensation,
\begin{equation}
\begin{aligned}
\min_{S_{t,c}^{k,(r)}}\; \mathbb{E}_{\theta'}\!\Big[\,
&\big\|\overline{\phi_{\theta'}}(S_{t,c}^{k,(r)})-
\overline{\phi_{\theta'}}(\D_{t,c}^k)\big\|^2\\
&+\big\|\overline{f_{\theta'}}(S_{t,c}^{k,(r)})-
\overline{f_{\theta'}}(\D_{t,c}^k)\big\|^2
\,\Big],
\end{aligned}
\label{eq:dm}
\end{equation}
where $\overline{(\cdot)}$ is the average over the set. Since this matching cannot be done for every model parameter, we take the
expectation over $\theta'$, which is $\theta_t^{(r-1)}$ perturbed by isotropic
Gaussian noise within a ball of radius $\rho$~\citep{Xiong_2023_CVPR}. Each surrogate set $S_{t,c}^{k,(r)}$ is initialized from real
images of the class and holds $\mathrm{IPC}$ images, or as many images as the
client has of the class when it has fewer.
The client then uploads the collected surrogates
$S_t^{k,(r)}=\bigcup_{c}S_{t,c}^{k,(r)}$ to the server.
In each condensation step, the real images only pass forward through this
model, and the gradient flows back only to the few surrogates without updating the
model. Since this runs for a few steps on a few images per class, condensation
takes less computation than the local training of conventional FL, which
back-propagates through every local sample and updates the model.

\begin{algorithm}[t]
\fontsize{9}{12.6}\selectfont
\SetAlCapFnt{\small}
\SetAlCapNameFnt{\normalfont\small}
\SetAlgoLined
\DontPrintSemicolon
\SetKwComment{Comment}{$\triangleright$\ }{}
\SetCommentSty{textnormal}
\KwIn{$N$ clients, $K$ per round, $T$ tasks, $R$ rounds per task, images per class
$\mathrm{IPC}$, budget $B$, temporal fraction $\lambda$}
initialize $\theta_0^{(0)}$\;
\For{$t = 0$ \KwTo $T-1$}{
  expand the classifier head by $|\Cset_t|$ outputs; $S_t \leftarrow \emptyset$\;
  \For{$r = 1$ \KwTo $R$}{
    sample participants $\mathcal{K}_{t,r}$ and broadcast $\theta_t^{(r-1)}$\;
    \ForEach{client $k\in\mathcal{K}_{t,r}$ \textnormal{in parallel}}{
      upload $S_t^{k,(r)} \leftarrow \textsc{Condense}(\D_t^k;\theta_t^{(r-1)})$
        \Comment*[r]{Eq.~\ref{eq:dm}}
    }
    $S_t \leftarrow S_t \cup \bigcup_{k\in\mathcal{K}_{t,r}} S_t^{k,(r)}$\;
    $\theta_t^{(r)} \leftarrow
      \textsc{ServerTrain}(\theta_t^{(r-1)},\B_{0:t-1}\cup S_t)$\;
  }
  \ForEach{class $c\in\Cset_t$}{
    $\B_{t,c} \leftarrow \textsc{TemporalHerding}(S_{t,c},\lambda,B;\theta_t^{(R)})$
      \Comment*[r]{Eq.~\ref{eq:herding}}
  }
  $\B_t \leftarrow \bigcup_{c\in\Cset_t}\B_{t,c}$; \quad $\theta_{t+1}^{(0)} \leftarrow \theta_t^{(R)}$\;
}
\caption{ReSCENE: server-side surrogate replay with temporal herding}
\label{alg:main}
\end{algorithm}

\noindent\textbf{Server-side replay.}\quad
The uploaded surrogates join the pool $S_t$ of the current task, and at the end of
task $t$ the server compresses this pool into a buffer $\B_t$
(Sec.~\ref{sec:temporal_herding}). In every round, the server trains the global
model on $\B_{0:t-1}\cup S_t$, so the classes of the earlier and current tasks are
optimized together.
Under non-IID data, a few classes dominate the local data of a client, but since
the client condenses at most $\mathrm{IPC}$ surrogates per class, it sends a more
balanced set of classes than it holds. The server gathers these sets from all
clients, so the replay covers every class that any client has held. The buffers
of earlier tasks and the pool of the current task still hold different numbers of
surrogates per class, so we apply square-root class sampling and
logit adjustment~\citep{menon2021longtail} to correct this imbalance
(Appendix~\ref{app:balance_logit_ablation}).

\noindent\textbf{Surrogate privacy.}\quad
The surrogates are the only data that leave a client, and
since the server only stores and replays them, the replay adds no privacy cost.
FedDM~\citep{Xiong_2023_CVPR} communicates the same surrogates, so its privacy
analysis applies to ReSCENE. By default, a client initializes the
surrogates from real images, so the surrogates and their class labels are not
covered by a formal privacy guarantee. For such a guarantee, we initialize the
surrogates from Gaussian noise and condense them under randomly initialized
networks with DP-SGD~\citep{10.1145/2976749.2978318}, so each upload is $(\epsilon,\delta)$-differentially
private. Appendix~\ref{app:privacy} gives this guarantee and its empirical analysis.

\subsection{Temporal Herding}
\label{sec:temporal_herding}

\noindent\textbf{Co-evolution of the model and the surrogates.}\quad
Within a task, the global model and the surrogates co-evolve. In each round, the
clients condense their data with the model trained on the earlier surrogates, and
the server trains this model further on the new ones. As the model improves over
the rounds, later surrogates are optimized under a better model and lie closer to
the real data under the task-end model.

We verify this co-evolution by tracking two quantities over the rounds of each
task, how well the model separates the classes of the task and how closely the
surrogates condensed under it represent the real data. The separability of a
model is given by the distance between the class means of its normalized
features divided by the average distance of the samples to their class mean.
The surrogates of every round are condensed from the same initial images, and
their distance to the real data is the gap between their mean feature and that
of the client's real data under the task-end model $\theta_t^{(R)}$, divided by
the same gap for the initial images. Since the surrogates co-evolve with the
model, we measure this gap under $\theta_t^{(R)}$, the model that the server keeps
training and in whose feature space temporal herding selects the buffer.
Figure~\ref{fig:coevolution} shows that, within every task, the separability of
the model rises over the rounds while the surrogates condensed with it move
closer to the real data.

\begin{figure}[t]
\centering
\includegraphics[width=0.62\linewidth]{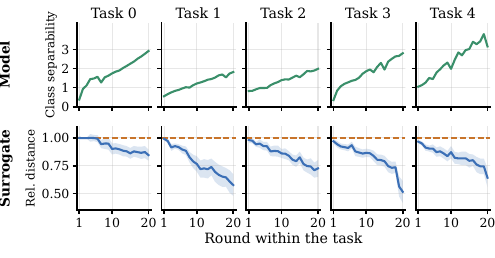}
\caption{Co-evolution of the global model and the surrogates within each task of
CIFAR-10 ($\beta{=}0.5$). Top, class separability of the model that the clients
receive in each round. Bottom, distance between the surrogates condensed under
that model and their client's real data, relative to that of the initial images
(dashed line). Shaded bands denote the standard error over
clients and classes.}
\label{fig:coevolution}
\end{figure}

\noindent\textbf{Temporal herding.}\quad
Based on this co-evolution, we introduce temporal herding, which compresses the
pool of a finished task into a smaller buffer that keeps the surrogates lying
closer to the real data. It takes the mean of the recent surrogates as the target and selects from
the whole pool by herding~\citep{10.1145/1553374.1553517}, so an earlier
surrogate also enters the buffer whenever it moves the buffer mean toward that
target.

Let $S_{t,c}$ be the class-$c$ part of the pool $S_t$, $\lambda\in(0,1]$ the
temporal fraction, the fraction of the $R$ rounds that forms the recent window, and
$W_{t,c}\subseteq S_{t,c}$ the class-$c$ surrogates uploaded in the last
$\lceil\lambda R\rceil$ rounds of the task. 
Under the task-end model
$\theta_t^{(R)}$ we normalize each surrogate feature,
$\vz(s)=\phi_{\theta_t^{(R)}}(s)/\|\phi_{\theta_t^{(R)}}(s)\|$, and take the mean
over the recent window, $\vm_{t,c}=|W_{t,c}|^{-1}\sum_{s\in W_{t,c}}\vz(s)$,
as the target. 
Herding then picks $B$ surrogates $p_1,\ldots,p_B$ from
$S_{t,c}$, where $p_\ell$ is chosen among the remaining candidates so that the mean
of $p_1,\ldots,p_\ell$ is closest to the target,
\begin{equation}
\scalebox{0.85}{$\displaystyle
p_\ell
\in\argmin_{s\in S_{t,c}\setminus\{p_1,\ldots,p_{\ell-1}\}}
\big\|
\vm_{t,c}
-\tfrac1\ell
\big(
\vz(s)+\sum\nolimits_{j<\ell}\vz(p_j)
\big)
\big\|^2
$},
\label{eq:herding}
\end{equation}
and the selected $B$ surrogates form $\B_{t,c}$. 
Collecting $\B_{t,c}$ over the classes of the task gives the final buffer
of the task, $\B_t=\bigcup_{c\in\Cset_t}\B_{t,c}$.

\subsection{Theoretical Analysis of Temporal Herding}
\label{sec:theory_short}

Full accumulation, which keeps the whole pool of a task without any selection,
is a natural choice for replay, since it discards none of the
surrogates. We establish a sufficient condition under which the
buffer formed by temporal herding represents the class more closely than full
accumulation, even though it keeps only $B$ surrogates per class.
We provide the full proofs in Appendix~\ref{app:temporal_target}.

\noindent\textbf{Setup.}\quad
We fix a class and write $P$ for the indices of all its surrogates in the task
pool. $P_r\subseteq P$ holds those uploaded in round $r$ and tracks how their error
changes over rounds, $W\subseteq P$ holds those from the recent
$\lceil\lambda R\rceil$ rounds, and $H_\ell\subseteq P$ holds those selected in
the first $\ell$ steps of Eq.~\eqref{eq:herding}, so that $H_B$ indexes the buffer
$\B_{t,c}$. Each surrogate is represented by its normalized task-end feature
$\vz_i$, and $\vm_A$ denotes the mean of $\vz_i$ over an index set $A$, so that
$\vm_W$ is the target $\vm_{t,c}$ of Eq.~\eqref{eq:herding}.
Let $\vmu(\theta)$ be the real class mean under model $\theta$, and write
$\vmu=\vmu(\theta_t^{(R)})$ for the mean of the real class under the task-end
model, the quantity that $\B_{t,c}$ should represent faithfully.
Our analysis is based on the following assumptions.
\begin{restated}{Assumption}{asm:directional_coverage}
\textup{(Directional coverage)} For each $2\le\ell\le B$, some
$i\in P\setminus H_{\ell-1}$ satisfies
\[
\langle\vm_W-\vm_{H_{\ell-1}},\,\vz_i-\vm_W\rangle\ge0.
\]
\end{restated}
Assumption~\ref{asm:directional_coverage} ensures that at every step some
candidate not yet selected moves the mean of the selected set toward the target,
since it lies at or beyond $\vm_W$ along the direction from the current mean to
the target.
\begin{restated}{Assumption}{asm:coherent_bias}
\textup{(Coherent stale bias)} There exist a unit vector $\vu$, coefficients
$a_r\ge0$, and vectors $\vxi_r$ with $\|\vxi_r\|\le\nu$
such that $\vm_{P_r}-\vmu=a_r\vu+\vxi_r$ for every round $r$.
\end{restated}
Assumption~\ref{asm:coherent_bias} decomposes the task-end error of the round-$r$
surrogates into a stale bias along a fixed direction and a bounded deviation.

Let $\Delta_W$ be the largest distance from a candidate to the target, which is at
most $2$ since the features are unit vectors. Let $A_P$ and $A_W$ be the weighted
averages of $a_r$ over the whole pool and over the recent window, each round
weighted by its number of surrogates. These quantities represent the stale biases
carried by $\vm_P$ and $\vm_W$, and $A_W\le A_P$ when the bias shrinks over the
rounds, since $W$ holds the latest surrogates.
\begin{restated}{Lemma}{thm:greedy_target_bound}
Under Assumption~\ref{asm:directional_coverage}, for $B\le|P|$, the buffer
$\B_{t,c}$ satisfies $\|\vm_{H_B}-\vm_W\|\le\Delta_W/\sqrt B$.
\end{restated}
Lemma~\ref{thm:greedy_target_bound} shows that the finite buffer $\B_{t,c}$
tracks the recent target, since each surrogate picked by Eq.~\eqref{eq:herding}
keeps the mean of the selected set close to it, and the approximation error
shrinks as $1/\sqrt B$.
\begin{restated}{Theorem}{thm:greedy_full_improvement}
Under Assumptions~\ref{asm:directional_coverage} and~\ref{asm:coherent_bias},
the buffer $\B_{t,c}$ is strictly closer to the real class mean than full
accumulation under the condition
\[
\underbrace{A_P-A_W}_{\text{removed stale bias}}
>\underbrace{2\nu}_{\text{deviation}}
+\underbrace{\Delta_W/\sqrt B}_{\text{approximation error}}
\]
\[
\Longrightarrow\;
\|\vm_{H_B}-\vmu\|<\|\vm_P-\vmu\|.
\]
\end{restated}
\begin{remark}
Theorem~\ref{thm:greedy_full_improvement} states that the buffer of temporal
herding represents the real class mean more closely than full accumulation
whenever the stale bias it removes outweighs the deviation of the rounds and the
approximation error of herding. Assumption~\ref{asm:directional_coverage} requires
that, at every step, some unselected surrogate is at or past the target when
viewed from the current buffer mean, so that herding can always move the buffer
mean in the direction of the target. Since the target is the average of the
recent surrogates, some of them are always at or past it from any viewpoint, so
the assumption holds as long as such a recent surrogate remains unselected.
Assumption~\ref{asm:coherent_bias} requires that the errors of the rounds share a
common direction. The early surrogates were matched under models farther from the
task-end model, so while the model keeps changing within the task, their stale
bias is expected to exceed that of the recent ones
(Proposition~\ref{lem:terminal_stale} in Appendix~\ref{app:temporal_target}),
which makes the removed bias positive. Figure~\ref{fig:coevolution} is consistent
with this, as the surrogates of the early rounds are farther from the real data
under the task-end model than those of the late rounds. The approximation error
is at most $2/\sqrt{1000}\approx0.06$ at our budget, so the condition reduces to
the removed bias exceeding the deviation by a margin of at most $0.06$.
Section~\ref{sec:herding_main} shows that the buffer is more accurate than the
whole pool on CIFAR-100 and TinyImageNet. A
narrower window removes more stale bias but averages fewer rounds and carries
more of their variation, so the temporal fraction $\lambda$ balances the two in
practice (Appendix~\ref{sec:policy_ablation}).
\end{remark}

\FloatBarrier

\section{Experiments}

\subsection{Experimental Setup}
\label{sec:setup}

\noindent\textbf{Datasets and Protocol.}\quad
We evaluate on three FCIL benchmarks. CIFAR-10 is divided
into $T{=}5$ two-class tasks, and CIFAR-100 and TinyImageNet into $T{=}10$ tasks
of ten and twenty classes, respectively. 
We use $N{=}20$ clients, $K{=}10$ participants per round, $R{=}20$
communication rounds per task, and task-wise Dirichlet label skew
$\beta\in\{0.1,0.5,1.0\}$. Every configuration is run with three seeds, and we
report the mean and standard deviation.

\noindent\textbf{Baselines and Implementation.}\quad
All methods use ResNet-18 with a classifier that expands as new tasks arrive.
We compare with FedAvg~\citep{pmlr-v54-mcmahan17a}, with
FedEWC~\citep{doi:10.1073/pnas.1611835114} and
FedLwF~\citep{10.1109/TPAMI.2017.2773081} as representatives of the
regularization and distillation families, and with
GLFC~\citep{Dong_2022_CVPR}, FedCBDR~\citep{NEURIPS2025_d611d06e}, and
Re-Fed~\citep{Li_2024_CVPR} as exemplar-based and
TARGET~\citep{Zhang_2023_ICCV} and MFCL~\citep{NEURIPS2023_d160ea01} as
generator-based replay methods. We further compare with the
surrogate-aggregation FL methods FedDM~\citep{Xiong_2023_CVPR} and FedAF~\citep{Wang_2024_CVPR}
under the same condensation settings as ReSCENE. Every baseline
trains its local model for $20$ epochs per round. ReSCENE uses IPC of $10$, a
per-class buffer budget $B{=}1000$, a temporal fraction $\lambda{=}0.75$, and
$E_{\mathrm{srv}}{=}100$ server epochs per round.
Appendix~\ref{app:hp} describes the datasets, settings, and per-method
hyperparameters in full detail. 

\noindent\textbf{Metrics.}\quad
We report Average Accuracy (AA), the accuracy over all seen classes after the
final task, and Average Incremental Accuracy (AIA), which measures that same
accuracy after every task and averages it over the $T$ tasks. AA captures what the
model retains at the end of the sequence, and AIA captures how well it performed throughout it.

\subsection{Main Results}

Table~\ref{tab:main} reports AA and AIA on the three benchmarks. ReSCENE
attains the highest AA in all nine settings and the highest AIA in seven of
them, reaching $67.96\%$ AA on CIFAR-10 and $49.80\%$ on CIFAR-100 at
$\beta{=}0.5$. It leads the generative replay methods by $4.6$ to
$51.8$ AA points, and it also leads the strongest exemplar-based replay method
by up to $31.1$ AA points on CIFAR-10 and by
$8.2$ to $10.7$ points on CIFAR-100. On TinyImageNet, where every method reaches
lower accuracy, ReSCENE still performs best in all settings except the AIA at
$\beta{=}0.5$ and $1.0$.
ReSCENE leads the strongest baseline by the widest margin at
$\beta{=}0.1$, where the data are most heterogeneous, on both CIFAR-10 and
CIFAR-100.
The advantage over the
replay baselines also holds when every replay method that stores replay data
receives the same budget as ReSCENE, where it stays $17.2$ to $48.3$ AA points ahead on CIFAR-10 at $\beta{=}0.5$
(Appendix~\ref{app:cost_memory}). We also
evaluate ReSCENE in the domain-incremental setting, where it shows the highest
accuracy on every domain (Appendix~\ref{sec:domain_plan}).

\begin{table*}[!t]
\centering

\footnotesize
\renewcommand{\arraystretch}{1.0}
\setlength{\tabcolsep}{2pt}
\begin{tabular*}{\textwidth}{@{\extracolsep{\fill}}llc cccccccc>{\columncolor{oursbg}}c@{}}
\toprule
Dataset & $\beta$ & Metric & FedAvg & FedEWC & FedLwF & TARGET & MFCL & GLFC & FedCBDR & Re-Fed & ReSCENE\\
\midrule
\multirow{6}{*}{CIFAR-10}
& \multirow{2}{*}{0.1} & AA  & 10.00\std{0.00}& 10.00\std{0.00}& 10.41\std{0.61}& 9.95\std{0.34}& 13.33\std{4.64} & 14.91\std{5.36}& 25.94\std{3.63}& 30.67\std{3.72}& \textbf{61.78}\std{1.14}\\
& & AIA & 23.34\std{0.88}& 23.37\std{0.92}& 22.93\std{0.09}& 22.80\std{0.04}& 24.30\std{1.89} & 24.65\std{2.41}& 37.23\std{3.36}& 36.74\std{5.41}& \textbf{68.35}\std{0.62}\\
\cmidrule(lr){2-12}
& \multirow{2}{*}{0.5} & AA  & 15.84\std{3.04}& 17.54\std{0.94}& 16.94\std{5.70}& 21.66\std{7.12}& 35.94\std{5.67} & 35.38\std{1.60}& 46.78\std{0.93}& 50.73\std{3.02}& \textbf{67.96}\std{1.33}\\
& & AIA & 32.90\std{7.78}& 36.81\std{5.03}& 33.57\std{11.77}& 36.76\std{13.40}& 46.70\std{10.52} & 48.90\std{6.83}& 57.64\std{4.59}& 62.69\std{6.83}& \textbf{73.53}\std{0.73}\\
\cmidrule(lr){2-12}
& \multirow{2}{*}{1.0} & AA  & 14.70\std{4.25}& 15.12\std{4.53}& 22.82\std{7.36}& 18.86\std{1.39}& 40.54\std{2.75} & 43.13\std{0.96}& 50.06\std{0.91}& 56.57\std{1.26}& \textbf{69.64}\std{0.84}\\
& & AIA & 30.83\std{1.89}& 31.78\std{4.52}& 33.51\std{5.96}& 32.37\std{3.58}& 43.58\std{5.59} & 52.56\std{5.15}& 57.07\std{2.63}& 66.83\std{2.33}& \textbf{74.57}\std{0.67}\\
\midrule
\multirow{6}{*}{CIFAR-100}
& \multirow{2}{*}{0.1} & AA  & 5.80\std{0.91}& 5.84\std{0.45}& 9.01\std{2.41}& 8.48\std{4.04}& 13.44\std{1.26} & 17.84\std{0.69}& 34.21\std{1.42}& 29.58\std{0.85}& \textbf{44.89}\std{1.13}\\
& & AIA & 9.93\std{0.83}& 10.05\std{0.84}& 11.84\std{3.56}& 10.53\std{2.48}& 13.37\std{0.67} & 26.18\std{1.24}& 37.97\std{0.75}& 34.98\std{0.77}& \textbf{49.67}\std{0.69}\\
\cmidrule(lr){2-12}
& \multirow{2}{*}{0.5} & AA  & 8.14\std{0.11}& 8.14\std{0.05}& 20.65\std{0.81}& 20.18\std{3.87}& 24.70\std{1.99} & 20.23\std{0.15}& 41.64\std{1.31}& 33.92\std{0.47}& \textbf{49.80}\std{0.22}\\
& & AIA & 18.46\std{0.22}& 18.83\std{0.79}& 28.62\std{1.24}& 30.95\std{2.11}& 35.40\std{3.71} & 31.67\std{0.52}& 49.83\std{0.45}& 49.89\std{0.10}& \textbf{57.71}\std{0.34}\\
\cmidrule(lr){2-12}
& \multirow{2}{*}{1.0} & AA  & 8.24\std{0.05}& 8.27\std{0.21}& 23.60\std{0.45}& 24.90\std{1.34}& 26.78\std{2.01} & 20.86\std{0.14}& 41.83\std{0.37}& 34.65\std{0.47}& \textbf{50.70}\std{0.17}\\
& & AIA & 20.44\std{0.92}& 20.78\std{0.81}& 34.54\std{2.78}& 35.13\std{2.86}& 38.14\std{2.95} & 32.85\std{0.29}& 52.03\std{1.55}& 51.94\std{0.56}& \textbf{58.85}\std{0.19}\\
\midrule
\multirow{6}{*}{TinyImageNet}
& \multirow{2}{*}{0.1} & AA  & 4.81\std{0.19}& 4.80\std{0.18}& 10.97\std{1.20}& 13.64\std{0.13}& 12.15\std{0.08}& 8.32\std{0.70}& 21.10\std{1.23}& 14.98\std{1.20}& \textbf{23.84}\std{0.18}\\
& & AIA & 11.38\std{0.15}& 11.26\std{0.63}& 16.66\std{1.01}& 18.86\std{0.87}& 17.22\std{0.41}& 16.45\std{0.57}& 27.27\std{1.49}& 26.75\std{0.90}& \textbf{31.80}\std{0.68}\\
\cmidrule(lr){2-12}
& \multirow{2}{*}{0.5} & AA  & 6.04\std{0.16}& 6.08\std{0.01}& 19.20\std{0.25}& 21.50\std{0.04}& 18.78\std{0.20}& 10.64\std{0.46}& 25.80\std{0.82}& 21.78\std{0.23}& \textbf{26.09}\std{1.01}\\
& & AIA & 17.21\std{0.29}& 17.23\std{0.28}& 29.06\std{1.84}& 30.78\std{0.50}& 30.05\std{1.50}& 21.10\std{0.29}& 34.45\std{1.87}& \textbf{37.96}\std{0.64}& 36.07\std{0.86}\\
\cmidrule(lr){2-12}
& \multirow{2}{*}{1.0} & AA  & 6.21\std{0.02}& 6.23\std{0.06}& 19.23\std{0.59}& 20.66\std{3.22}& 19.32\std{0.95}& 10.82\std{0.68}& 26.10\std{1.17}& 22.91\std{0.40}& \textbf{28.69}\std{0.43}\\
& & AIA & 17.39\std{0.29}& 17.41\std{0.17}& 29.26\std{1.15}& 30.56\std{4.35}& 31.39\std{1.41}& 22.22\std{0.28}& 35.35\std{3.59}& \textbf{39.40}\std{0.31}& 37.86\std{0.13}\\
\bottomrule
\end{tabular*}
\caption{AA and AIA (\%) of the baselines and ReSCENE. Small gray values are
standard deviations over three seeds.}
\label{tab:main}
\end{table*}

\noindent\textbf{Effect of surrogate retention.}\quad
We compare ReSCENE with other surrogate-aggregation FL methods in the FCIL setting. Both methods follow the same condense-and-upload structure as
ReSCENE, but the server trains only on the surrogates received in the current
round and discards them afterwards. Table~\ref{tab:safl} shows that both
methods fail to retain any earlier task. Since their server keeps no surrogates of the earlier classes, each new task
overwrites those classes, whereas ReSCENE replays the kept surrogates together with
the current task, which points to this server-side replay as the source of the gap.

\begin{table}[t]
\centering

\footnotesize
\setlength{\tabcolsep}{6pt}
\renewcommand{\arraystretch}{1.31}
\begin{tabular}{@{}ll cc>{\columncolor{oursbg}[1pt][1pt]}c@{\hspace{1pt}}}
\toprule
Dataset & & FedDM & FedAF & ReSCENE\\
\midrule
\multirow{2}{*}{CIFAR-10}  & AA & 15.69\std{1.12} & 15.99\std{0.92} & \textbf{61.78}\std{1.14}\\
                           & AIA & 37.06\std{0.47} & 37.06\std{1.05} & \textbf{68.35}\std{0.62}\\
\midrule
\multirow{2}{*}{CIFAR-100} & AA & 6.27\std{0.92} & 6.37\std{0.76} & \textbf{44.89}\std{1.13}\\
                           & AIA & 16.31\std{0.83} & 16.35\std{1.41} & \textbf{49.67}\std{0.69}\\
\midrule
\multirow{2}{*}{TinyImageNet} & AA & 2.52\std{0.82} & 2.58\std{0.54} & \textbf{23.84}\std{0.18}\\
                           & AIA & 10.03\std{1.36} & 10.37\std{0.91} & \textbf{31.80}\std{0.68}\\
\bottomrule
\end{tabular}
\caption{AA and AIA (\%) at $\beta{=}0.1$ of the surrogate-aggregation FL methods, which keep no earlier surrogates, and ReSCENE.}
\label{tab:safl}
\end{table}

\subsection{Forgetting Analysis}
\label{sec:analysis}

Catastrophic forgetting is a central challenge in FCIL, as knowledge acquired on
earlier tasks is erased when learning new tasks, so mitigating it requires
preserving task-level accuracy throughout the sequence. We split the task
sequence into an old and a recent group, measure the accuracy of each group
after the entire sequence, and plot them against each other in
Figure~\ref{fig:old_recent}. For CIFAR-10 the old group is tasks 0 and 1 and the
recent group tasks 2 to 4, and for the ten-task datasets each group holds five
tasks. Appendix~\ref{app:forgetting} reports the per-task accuracy behind every
entry of Table~\ref{tab:main} and plots the full task-wise profiles.

\begin{figure}[t!]
\centering
\includegraphics[width=0.62\linewidth,trim=0 11 0 0,clip]{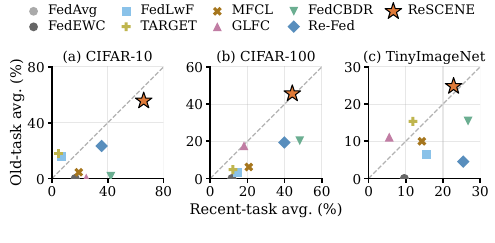}
\caption{Old-task versus recent-task accuracy at $\beta{=}0.1$. Points high on
both axes and near the diagonal forget less, performing well on both old and
recent tasks.}
\label{fig:old_recent}
\end{figure}

Most baselines lie below the diagonal, so their final accuracy relies on the
recent tasks while the old tasks are largely forgotten. ReSCENE attains the
highest old-task accuracy on every benchmark, and lies close to the diagonal, showing
comparable accuracy on old and recent tasks. This balance results from the
server keeping the surrogates of every earlier task and training on them
together with the current one, which allows ReSCENE to retain earlier-task
knowledge throughout the sequence.

\subsection{Client-Side Cost}
\label{sec:client_cost}
\newcommand{\dup}[1]{{\scriptsize\textcolor[HTML]{2E7D32}{$\uparrow$#1}}}
\newcommand{\ddn}[1]{{\scriptsize\textcolor[HTML]{B5473A}{$\downarrow$#1}}}

In ReSCENE, a client only condenses its current data and stores nothing from
earlier tasks, so continual learning should add no cost to it. We therefore
measure the client-side cost of every method on the three benchmarks under the
setting of Table~\ref{tab:main}. Table~\ref{tab:client_cost} reports the
computation that the clients spend over a full run and the upload of a client
per round, and Appendix~\ref{app:cost} gives the full accounting for
all methods together with the server cost.

\begin{table}[t]
\centering

\footnotesize
\setlength{\tabcolsep}{5pt}
\renewcommand{\arraystretch}{0.93}
\begin{tabular}{@{}l cc|cc|cc@{}}
\toprule
& \multicolumn{2}{c|}{CIFAR-10} & \multicolumn{2}{c|}{CIFAR-100} & \multicolumn{2}{c}{TinyImageNet}\\
Method & Comp. & Uplink & Comp. & Uplink & Comp. & Uplink\\
\midrule
FedAvg  & 32.68 & 44.73 & 33.24 & 44.92 & 67.21 & 45.12\\
FedEWC  & 34.31 & 44.73 & 34.90 & 44.92 & 70.58 & 45.12\\
FedLwF  & 41.32 & 44.73 & 43.20 & 44.92 & 87.38 & 45.12\\
TARGET  & 72.20 & 44.73 & 84.30 & 44.92 & 158.60 & 45.12\\
MFCL    & 111.34 & 44.73 & 132.45 & 44.92 & 264.96 & 45.12\\
GLFC    & 48.43 & 44.73 & 203.18 & 44.92 & 410.42 & 45.12\\
FedCBDR & 72.71 & 44.73 & 933.15 & 44.92 & 1884.40 & 45.12\\
Re-Fed  & 71.70 & 44.73 & 146.46 & 44.92 & 295.96 & 45.12\\
\rowcolor{oursbg}ReSCENE & \textbf{3.59} & \textbf{0.25} & \textbf{27.49} & \textbf{1.23} & \textbf{56.20} & \textbf{9.83}\\
\bottomrule
\end{tabular}
\caption{Client computation over a full run (PFLOP) and uplink per round (MB) of
the baselines and ReSCENE.}
\label{tab:client_cost}
\end{table}

\noindent\textbf{Computation and storage.}\quad
The replay baselines raise the client computation of FedAvg by $1.5$ to
$3.4\times$ on CIFAR-10 and by $2.4$ to $28.1\times$ on CIFAR-100 and
TinyImageNet, since every client trains on the replay data in addition to the
new data. ReSCENE requires $0.11\times$, $0.83\times$, and $0.84\times$ of the
computation of FedAvg on the three benchmarks, which is the lowest of all
methods. Its clients also hold only the received model between rounds, the same state as
in FedAvg, whereas the replay baselines hold $1.3$ to $33.7\times$ of it for
their exemplars, synthetic data, or earlier models (Appendix~\ref{app:cost}).

\noindent\textbf{Communication.}\quad
Every baseline uploads the full model, which takes $44.7$ to $45.1$~MB per
round. In contrast, ReSCENE uploads a set of surrogates, which takes only
$0.25$, $1.23$, and $9.83$~MB on CIFAR-10,
CIFAR-100, and TinyImageNet, which is $179\times$, $36.5\times$, and
$4.6\times$ smaller.

\section{Ablations and Analyses}
\label{sec:analyses}

\subsection{Scalability to Larger Client Populations and Models}
\label{sec:scaling_clients}

In practice, federated learning often involves many clients and larger models,
so we test ReSCENE on both.

\noindent\textbf{Larger client populations.}\quad
Cross-device federated learning involves a large population of clients, each
holding a small share of the data~\citep{10.1561/2200000083}. To move closer
to this regime, we scale the population to $N{=}100$ clients with $K{=}70$
participants per round on CIFAR-10 at $\beta{=}0.5$. Every baseline loses AA as the population
grows, whereas ReSCENE gains $4.17$ AA and $4.93$ AIA points over the $N{=}20$
setting and reaches $72.13$ AA (Table~\ref{tab:scaling_clients}). With more
clients, each client holds less and more skewed data, so each local update fits
a narrower distribution, and averaging many such updates degrades the global
model. In ReSCENE, each client condenses its
own data, and the server trains on the surrogates of all $70$ participants together, so a larger population supplies surrogates from more partitions,
while temporal herding keeps the buffer at $B{=}1000$ per class. ReSCENE
keeps the lowest client cost, $0.47\times$ FedAvg's computation and $0.6\%$ of
its uplink.

\begin{table}[t]
\centering

\footnotesize
\setlength{\tabcolsep}{5pt}
\renewcommand{\arraystretch}{1.11}
{\begin{tabular}{l cc cc}
\toprule
& Client comp. & Uplink & \multicolumn{2}{c}{Accuracy}\\
Method & (PFLOP/run) & (MB/rd) & AA & AIA\\
\midrule
FedAvg  & 46.33 & 44.73 & 14.58\,\ddn{1.26} & 28.49\,\ddn{4.41}\\
FedEWC  & 48.65 & 44.73 & 15.97\,\ddn{1.57} & 31.62\,\ddn{5.19}\\
FedLwF  & 58.69 & 44.73 & 14.38\,\ddn{2.56} & 29.30\,\ddn{4.27}\\
TARGET  & 132.25 & 44.73 & 16.01\,\ddn{5.65} & 31.29\,\ddn{5.47}\\
MFCL    & 204.43 & 44.73 & 20.32\,\ddn{15.62} & 29.84\,\ddn{16.86}\\
GLFC   & 108.45 & 44.73 & 34.52\,\ddn{0.86} & 49.21\,\dup{0.31}\\
FedCBDR & 326.31 & 44.73 & 40.07\,\ddn{6.71} & 52.12\,\ddn{5.52}\\
Re-Fed  & 139.87 & 44.73 & 47.58\,\ddn{3.15} & 56.55\,\ddn{6.14}\\
\rowcolor{oursbg}ReSCENE & 21.93 & 0.25 & \textbf{72.13}\,\dup{4.17} & \textbf{78.46}\,\dup{4.93}\\
\bottomrule
\end{tabular}}
\caption{Accuracy and client cost with more clients (CIFAR-10, $\beta{=}0.5$).
The arrows and the numbers beside them give the change from $N{=}20$ in
Table~\ref{tab:main}.}
\label{tab:scaling_clients}
\end{table}

\noindent\textbf{Larger models.}\quad
We compare ReSCENE with FedCBDR and Re-Fed, the strongest baselines in
Table~\ref{tab:main}, on ResNet-50 and ViT-Small/4 for CIFAR-100 at
$\beta{=}0.5$, with ResNet-18 as the reference. ReSCENE attains the highest
accuracy at every model size, leading the strongest baseline by $8.2$ to $18.8$
AA points (Table~\ref{tab:scaling_models}). From ResNet-18 to ResNet-50, ReSCENE
gains $5.45$ AA points, while FedCBDR and Re-Fed lose $5.20$ and $4.86$.
FedCBDR and Re-Fed fit the larger model on each client's small and skewed local
data, too little for the added parameters. ReSCENE trains the model
at the server on the surrogates of all clients, which cover every class seen so
far, so the larger model benefits from its added capacity. Larger backbones also
raise the computation and uplink of FedCBDR and Re-Fed clients, which train and
upload the full model. ReSCENE clients condense instead of training
and upload surrogates in place of the model, so their
computation stays at $2.9\%$ of FedCBDR's and $18.8\%$ of Re-Fed's, and their
uplink is $1.3\%$ of the baselines' on ResNet-50. All methods drop on ViT-Small/4, which
lacks a convolutional inductive bias, yet ReSCENE still leads.

\begin{table}[t]
\centering

\footnotesize
\setlength{\tabcolsep}{4.6pt}
\renewcommand{\arraystretch}{1.12}
{\begin{tabular}{llcccc}
\toprule
Backbone & Method & Client comp. & Uplink & \multicolumn{2}{c}{Accuracy}\\
& & (PFLOP/run) & (MB/rd) & AA & AIA\\
\midrule
ResNet-18 & FedCBDR & 933.15 & 44.92 & 41.64 & 49.83\\
& Re-Fed & 146.46 & 44.92 & 33.92 & 49.89\\
& \ours ReSCENE & \ours 27.49 & \ours 1.23 & \ours \textbf{49.80} & \ours \textbf{57.71}\\
\midrule
ResNet-50 & FedCBDR & 2180.66 & 95.03 & 36.44 & 42.05\\
& Re-Fed & 342.34 & 95.03 & 29.06 & 38.54\\
& \ours ReSCENE & \ours 64.26 & \ours 1.23 & \ours \textbf{55.25} & \ours \textbf{61.70}\\
\midrule
ViT-Small/4 & FedCBDR & 2320.76 & 85.51 & 20.25 & 29.20\\
& Re-Fed & 364.34 & 85.51 & 22.66 & 33.27\\
& \ours ReSCENE & \ours 68.40 & \ours 1.23 & \ours \textbf{31.45} & \ours \textbf{38.59}\\
\bottomrule
\end{tabular}}
\caption{Accuracy and client cost with larger models (CIFAR-100, $\beta{=}0.5$).}
\label{tab:scaling_models}
\end{table}

\subsection{Herding-Buffer Ablation}
\label{sec:herding_main}

Following Sec.~\ref{sec:theory_short}, we compare temporal herding with full
accumulation and with other buffers of the same budget, and we ablate its
hyperparameters in Appendix~\ref{sec:policy_ablation}.

\begin{table}[t]
\centering

\footnotesize
\setlength{\tabcolsep}{1.8pt}
\renewcommand{\arraystretch}{1.15}
{\begin{tabular}{@{}llcccc@{}}
\toprule
Dataset & Method & \# Samples & Server comp. & AA & AIA\\
\midrule
CIFAR-10 & Full Accml. & 16,763 ($\times 1.00$) & $\times 1.00$ & \textbf{69.19} & \textbf{74.67}\\
& \ours ReSCENE & \ours 10,000 ($\times 0.60$) & \ours $\times 0.69$ & \ours 67.96 & \ours 73.53\\
\midrule
CIFAR-100 & Full Accml. & 123,738 ($\times 1.00$) & $\times 1.00$ & 47.54 & 56.81\\
& \ours ReSCENE & \ours 98,087 ($\times 0.79$) & \ours $\times 0.78$ & \ours \textbf{49.80} & \ours \textbf{57.71}\\
\midrule
TinyImageNet & Full Accml. & 250,528 ($\times 1.00$) & $\times 1.00$ & 20.02 & 33.18\\
& \ours ReSCENE & \ours 194,355 ($\times 0.78$) & \ours $\times 0.78$ & \ours \textbf{26.09} & \ours \textbf{36.07}\\
\bottomrule
\end{tabular}}
\vspace{0pt}
\caption{Comparison with full accumulation.}
\label{tab:herding_full}
\end{table}

\noindent\textbf{Comparison with full accumulation.}\quad
Full accumulation, which keeps every uploaded surrogate and trains on the largest
pool available, is a natural choice when server memory allows. We compare it
with ReSCENE on the three benchmarks at $\beta{=}0.5$. Although full
accumulation keeps a small edge of $1.23$ AA points on CIFAR-10, ReSCENE shows
higher accuracy on CIFAR-100 and TinyImageNet, by $2.26$ and $6.07$ AA points, while holding only $0.60$ to $0.79\times$ the surrogates and requiring $0.69$ to $0.78\times$
the server computation (Table~\ref{tab:herding_full}). 
Temporal herding drops part of the stale early surrogates that full
accumulation keeps, and Theorem~\ref{thm:greedy_full_improvement} gives a sufficient
condition under which this brings the buffer closer to the real class mean than
the whole pool. These results therefore show that temporal herding can build a better
buffer even with less server memory and computation.

\noindent\textbf{Buffer construction under a fixed budget.}\quad
We compare temporal herding with four other ways of building the buffer under the
same budget of $B{=}1000$ on CIFAR-10 at $\beta{=}0.5$, which keep the earliest,
the latest, or random surrogates, or herd toward the mean of the whole pool
gathered from every round. In Table~\ref{tab:herding_main}, temporal herding attains the highest AA
and AIA among the five in this setting.
Latest-1K exceeds Earliest-1K, the least accurate of the five, by $1.76$ AA points, which agrees with the
co-evolution, as the surrogates of later rounds represent the real data more
closely.
Latest-1K relies on the recent surrogates being less stale and keeps only the
newest ones, discarding all earlier surrogates.
Temporal herding instead uses the recent
surrogates to define the target and selects from the whole pool, adding an
earlier surrogate whenever it moves the buffer mean toward that target. 
It thus improves on Latest-1K by $1.40$ AA points.
The higher accuracy of temporal herding over full-pool herding is consistent with
Proposition~\ref{lem:terminal_stale}, which bounds the task-end error of a round's
surrogates by a staleness term that grows with the training after that round, so
the recent target keeps the staler early surrogates from biasing the buffer.

\begin{table}[t]
\centering
\footnotesize
\begin{minipage}[t]{0.49\linewidth}
\centering

\setlength{\tabcolsep}{3pt}
\renewcommand{\arraystretch}{1.24}
\begin{tabular}{@{}lcc@{}}
\toprule
Method & AA & AIA\\
\midrule
Earliest-1K & 64.80 & 69.82\\
Latest-1K & 66.56 & 72.67\\
Random-1K & 65.28 & 70.81\\
Full-Pool Herd. & 66.70 & 72.86\\
\rowcolor{oursbg}\textbf{Temporal Herd.} & \textbf{67.96} & \textbf{73.53}\\
\bottomrule
\end{tabular}
\caption{Buffer construction at $B{=}1000$ (CIFAR-10, $\beta{=}0.5$).}
\label{tab:herding_main}
\end{minipage}\hfill
\begin{minipage}[t]{0.49\linewidth}
\centering

\setlength{\tabcolsep}{3pt}
\renewcommand{\arraystretch}{1.15}
\begin{tabular}{@{}cccc@{}}
\toprule
$E_{\mathrm{srv}}$ & Server comp. & AA & AIA\\
\midrule
100 & 1.00$\times$ & 61.78 & 68.35\\
50  & 0.50$\times$ & 58.26 & 66.91\\
25  & 0.25$\times$ & 53.91 & 64.81\\
10  & 0.10$\times$ & 53.45 & 65.03\\
\midrule
\multicolumn{2}{@{}l}{Re-Fed} & 30.67 & 36.74\\
\bottomrule
\end{tabular}
\caption{Server epochs per round (CIFAR-10, $\beta{=}0.1$).}
\label{tab:esrv_main}
\end{minipage}
\end{table}

\subsection{Server Training Budget}
\label{sec:server_budget_main}
In this work, we assume a resource-rich server. Since server computation is still
not free, we ablate the number of server epochs $E_{\mathrm{srv}}$ per round, the
main factor of the server computation of ReSCENE, reducing it from $100$ to $50$,
$25$, and $10$ on CIFAR-10 at $\beta{=}0.1$. Cutting $E_{\mathrm{srv}}$ to $10$
removes $90\%$ of the server computation, whereas AA drops by $8.33$ points and
AIA by $3.32$ points (Table~\ref{tab:esrv_main}). At $E_{\mathrm{srv}}{=}10$,
ReSCENE still leads Re-Fed, the strongest baseline in this setting, by $22.78$ AA
points, so its server computation can be reduced substantially while it stays
ahead of the baselines.

\section{Conclusion}
\label{sec:conclusion}

We present ReSCENE, which moves replay in federated class-incremental learning
from the clients to the server. Clients upload condensed surrogates of their local
data, and the server keeps those of earlier tasks from all clients and trains the
global model on them together with the current task, so continual learning adds
no storage or computation to the clients and the replay is balanced across the
earlier classes. Building on the co-evolution of the global model and the
surrogates, temporal herding compresses each task into a smaller buffer, and our
analysis establishes a condition under which this buffer represents the task
more closely than keeping all surrogates. Across three benchmarks, ReSCENE attains
the highest average accuracy with less client computation and upload than the
baselines, and keeps this advantage with more clients and larger models. Overall,
ReSCENE shows that moving replay to the server is an effective way to counter
forgetting in federated class-incremental learning while keeping client cost low.

\noindent\textbf{Limitations and trade-offs.}\quad
ReSCENE assumes synchronous rounds with a shared task sequence. Beyond this, its
central trade-off is where the model is trained. Existing methods aggregate models trained on
each client's skewed local data, which loses part of what each client learns. ReSCENE avoids this loss by training a single
model at a resource-rich server, at the price of more server computation and
memory (Appendix~\ref{app:cost}) and of training on surrogates that only
approximate the real data. Our results suggest that this approximation costs less
than the aggregation it replaces, and better
condensation directly benefits ReSCENE.

{\small
\bibliographystyle{plainnat}
\bibliography{refs}

@article{10.1561/2200000083,
author = {Kairouz, Peter and McMahan, H. Brendan and Avent, Brendan and Bellet, Aur{\'e}lien and Bennis, Mehdi and Nitin Bhagoji, Arjun and Bonawitz, Kallista and Charles, Zachary and Cormode, Graham and Cummings, Rachel and D’Oliveira, Rafael G. L. and Eichner, Hubert and El Rouayheb, Salim and Evans, David and Gardner, Josh and Garrett, Zachary and Gasc{\'o}n, Adri{\`a} and Ghazi, Badih and Gibbons, Phillip B. and Gruteser, Marco and Harchaoui, Zaid and He, Chaoyang and He, Lie and Huo, Zhouyuan and Hutchinson, Ben and Hsu, Justin and Jaggi, Martin and Javidi, Tara and Joshi, Gauri and Khodak, Mikhail and Konecn{\'y}, Jakub and Korolova, Aleksandra and Koushanfar, Farinaz and Koyejo, Sanmi and Lepoint, Tancr{\`e}de and Liu, Yang and Mittal, Prateek and Mohri, Mehryar and Nock, Richard and {\"O}zg{\"u}r, Ayfer and Pagh, Rasmus and Qi, Hang and Ramage, Daniel and Raskar, Ramesh and Raykova, Mariana and Song, Dawn and Song, Weikang and Stich, Sebastian U. and Sun, Ziteng and Suresh, Ananda Theertha and Tram{\`e}r, Florian and Vepakomma, Praneeth and Wang, Jianyu and Xiong, Li and Xu, Zheng and Yang, Qiang and Yu, Felix X. and Yu, Han and Zhao, Sen},
title = {Advances and Open Problems in Federated Learning},
year = {2021},
issue_date = {Jun 2021},
publisher = {Now Publishers Inc.},
address = {Hanover, MA, USA},
volume = {14},
number = {1–2},
issn = {1935-8237},
url = {https://doi.org/10.1561/2200000083},
doi = {10.1561/2200000083},
journal = {Found. Trends Mach. Learn.},
month = jun,
pages = {1–210},
numpages = {214}
}

@InProceedings{pmlr-v54-mcmahan17a,
  title = 	 {{Communication-Efficient Learning of Deep Networks from Decentralized Data}},
  author = 	 {McMahan, Brendan and Moore, Eider and Ramage, Daniel and Hampson, Seth and Arcas, Blaise Aguera y},
  booktitle = 	 {Proceedings of the 20th International Conference on Artificial Intelligence and Statistics},
  pages = 	 {1273--1282},
  year = 	 {2017},
  editor = 	 {Singh, Aarti and Zhu, Jerry},
  volume = 	 {54},
  series = 	 {Proceedings of Machine Learning Research},
  month = 	 {20--22 Apr},
  publisher =    {PMLR},
  url = 	 {https://proceedings.mlr.press/v54/mcmahan17a.html}
}

@InProceedings{Xiong_2023_CVPR,
    author    = {Xiong, Yuanhao and Wang, Ruochen and Cheng, Minhao and Yu, Felix and Hsieh, Cho-Jui},
    title     = {FedDM: Iterative Distribution Matching for Communication-Efficient Federated Learning},
    booktitle = {Proceedings of the IEEE/CVF Conference on Computer Vision and Pattern Recognition (CVPR)},
    month     = {June},
    year      = {2023},
    pages     = {16323-16332}
}

@InProceedings{Wang_2024_CVPR,
    author    = {Wang, Yuan and Fu, Huazhu and Kanagavelu, Renuga and Wei, Qingsong and Liu, Yong and Goh, Rick Siow Mong},
    title     = {An Aggregation-Free Federated Learning for Tackling Data Heterogeneity},
    booktitle = {Proceedings of the IEEE/CVF Conference on Computer Vision and Pattern Recognition (CVPR)},
    month     = {June},
    year      = {2024},
    pages     = {26233-26242}
}

@InProceedings{Rebuffi_2017_CVPR,
author = {Rebuffi, Sylvestre-Alvise and Kolesnikov, Alexander and Sperl, Georg and Lampert, Christoph H.},
title = {iCaRL: Incremental Classifier and Representation Learning},
booktitle = {Proceedings of the IEEE Conference on Computer Vision and Pattern Recognition (CVPR)},
month = {July},
year = {2017}
}

@article{
doi:10.1073/pnas.1611835114,
author = {James Kirkpatrick  and Razvan Pascanu  and Neil Rabinowitz  and Joel Veness  and Guillaume Desjardins  and Andrei A. Rusu  and Kieran Milan  and John Quan  and Tiago Ramalho  and Agnieszka Grabska-Barwinska  and Demis Hassabis  and Claudia Clopath  and Dharshan Kumaran  and Raia Hadsell },
title = {Overcoming catastrophic forgetting in neural networks},
journal = {Proceedings of the National Academy of Sciences},
volume = {114},
number = {13},
pages = {3521-3526},
year = {2017},
doi = {10.1073/pnas.1611835114},
URL = {https://www.pnas.org/doi/abs/10.1073/pnas.1611835114},
eprint = {https://www.pnas.org/doi/pdf/10.1073/pnas.1611835114}}

@article{10.1109/TPAMI.2017.2773081,
author = {Li, Zhizhong and Hoiem, Derek},
title = {Learning without Forgetting},
year = {2018},
issue_date = {Dec. 2018},
publisher = {IEEE Computer Society},
address = {USA},
volume = {40},
number = {12},
issn = {0162-8828},
url = {https://doi.org/10.1109/TPAMI.2017.2773081},
doi = {10.1109/TPAMI.2017.2773081},
journal = {IEEE Trans. Pattern Anal. Mach. Intell.},
month = dec,
pages = {2935--2947},
numpages = {13}
}

@InProceedings{Dong_2022_CVPR,
    author    = {Dong, Jiahua and Wang, Lixu and Fang, Zhen and Sun, Gan and Xu, Shichao and Wang, Xiao and Zhu, Qi},
    title     = {Federated Class-Incremental Learning},
    booktitle = {Proceedings of the IEEE/CVF Conference on Computer Vision and Pattern Recognition (CVPR)},
    month     = {June},
    year      = {2022},
    pages     = {10164-10173}
}

@InProceedings{Zhang_2023_ICCV,
    author    = {Zhang, Jie and Chen, Chen and Zhuang, Weiming and Lyu, Lingjuan},
    title     = {TARGET: Federated Class-Continual Learning via Exemplar-Free Distillation},
    booktitle = {Proceedings of the IEEE/CVF International Conference on Computer Vision (ICCV)},
    month     = {October},
    year      = {2023},
    pages     = {4782-4793}
}

@InProceedings{Li_2024_CVPR,
    author    = {Li, Yichen and Li, Qunwei and Wang, Haozhao and Li, Ruixuan and Zhong, Wenliang and Zhang, Guannan},
    title     = {Towards Efficient Replay in Federated Incremental Learning},
    booktitle = {Proceedings of the IEEE/CVF Conference on Computer Vision and Pattern Recognition (CVPR)},
    month     = {June},
    year      = {2024},
    pages     = {12820-12829}
}

@inproceedings{NEURIPS2025_d611d06e,
 author = {Qi, Zhuang and Tang, Ying-Peng and Meng, Lei and Yu, Han and Li, Xiaoxiao and Meng, Xiangxu},
 booktitle = {Advances in Neural Information Processing Systems},
 doi = {10.52202/085713-4863},
 editor = {D. Belgrave and C. Zhang and H. Lin and R. Pascanu and P. Koniusz and M. Ghassemi and N. Chen},
 pages = {145309--145334},
 publisher = {Curran Associates, Inc.},
 title = {Class-wise Balancing Data Replay for Federated Class-Incremental Learning},
 url = {https://proceedings.neurips.cc/paper_files/paper/2025/file/d611d06e3207330555fbc10810e70163-Paper-Conference.pdf},
 volume = {38, Main Conference},
 year = {2025}
}

@inproceedings{
menon2021longtail,
title={Long-tail learning via logit adjustment},
author={Aditya Krishna Menon and Sadeep Jayasumana and Ankit Singh Rawat and Himanshu Jain and Andreas Veit and Sanjiv Kumar},
booktitle={International Conference on Learning Representations},
year={2021},
url={https://openreview.net/forum?id=37nvvqkCo5}
}

@inproceedings{NIPS2017_0efbe980,
 author = {Shin, Hanul and Lee, Jung Kwon and Kim, Jaehong and Kim, Jiwon},
 booktitle = {Advances in Neural Information Processing Systems},
 editor = {I. Guyon and U. Von Luxburg and S. Bengio and H. Wallach and R. Fergus and S. Vishwanathan and R. Garnett},
 pages = {},
 publisher = {Curran Associates, Inc.},
 title = {Continual Learning with Deep Generative Replay},
 url = {https://proceedings.neurips.cc/paper_files/paper/2017/file/0efbe98067c6c73dba1250d2beaa81f9-Paper.pdf},
 volume = {30},
 year = {2017}
}

@InProceedings{pmlr-v139-yoon21b,
  title = 	 {Federated Continual Learning with Weighted Inter-client Transfer},
  author =       {Yoon, Jaehong and Jeong, Wonyong and Lee, Giwoong and Yang, Eunho and Hwang, Sung Ju},
  booktitle = 	 {Proceedings of the 38th International Conference on Machine Learning},
  pages = 	 {12073--12086},
  year = 	 {2021},
  editor = 	 {Meila, Marina and Zhang, Tong},
  volume = 	 {139},
  series = 	 {Proceedings of Machine Learning Research},
  month = 	 {18--24 Jul},
  publisher =    {PMLR},
  url = 	 {https://proceedings.mlr.press/v139/yoon21b.html}
}

@misc{sun2025exemplarcondensedfederatedclassincrementallearning,
      title={Exemplar-condensed Federated Class-incremental Learning}, 
      author={Rui Sun and Yumin Zhang and Varun Ojha and Tejal Shah and Haoran Duan and Bo Wei and Rajiv Ranjan},
      year={2025},
      eprint={2412.18926},
      archivePrefix={arXiv},
      primaryClass={cs.LG},
      url={https://arxiv.org/abs/2412.18926}, 
}

@inproceedings{ijcai2022p303,
  title     = {Continual Federated Learning Based on Knowledge Distillation},
  author    = {Ma, Yuhang and Xie, Zhongle and Wang, Jue and Chen, Ke and Shou, Lidan},
  booktitle = {Proceedings of the Thirty-First International Joint Conference on
               Artificial Intelligence, {IJCAI-22}},
  publisher = {International Joint Conferences on Artificial Intelligence Organization},
  editor    = {Lud De Raedt},
  pages     = {2182--2188},
  year      = {2022},
  month     = {7},
  note      = {Main Track},
  doi       = {10.24963/ijcai.2022/303},
  url       = {https://doi.org/10.24963/ijcai.2022/303},
}

@inproceedings{10.1145/1553374.1553517,
author = {Welling, Max},
title = {Herding dynamical weights to learn},
year = {2009},
isbn = {9781605585161},
publisher = {Association for Computing Machinery},
address = {New York, NY, USA},
url = {https://doi.org/10.1145/1553374.1553517},
doi = {10.1145/1553374.1553517},
booktitle = {Proceedings of the 26th Annual International Conference on Machine Learning},
pages = {1121--1128},
numpages = {8},
location = {Montreal, Quebec, Canada},
series = {ICML '09}
}

@inproceedings{ICLR2024_cf53f12a,
 author = {Bakman, Yavuz Faruk and Yaldiz, Duygu Nur and Ezzeldin, Yahya and Avestimehr, Salman},
 booktitle = {International Conference on Learning Representations},
 editor = {B. Kim and Y. Yue and S. Chaudhuri and K. Fragkiadaki and M. Khan and Y. Sun},
 pages = {47351--47375},
 title = {Federated Orthogonal Training: Mitigating Global Catastrophic Forgetting in Continual Federated Learning},
 url = {https://proceedings.iclr.cc/paper_files/paper/2024/file/cf53f12ab4ea81d942f16162f7f37e6c-Paper-Conference.pdf},
 volume = {2024},
 year = {2024}
}

@inproceedings{Saenko2010AdaptingVC,
  title={Adapting Visual Category Models to New Domains},
  author={Kate Saenko and Brian Kulis and Mario Fritz and Trevor Darrell},
  booktitle={European Conference on Computer Vision},
  year={2010},
  url={https://api.semanticscholar.org/CorpusID:7534823}
}

@inproceedings{10.1145/2976749.2978318,
author = {Abadi, Martin and Chu, Andy and Goodfellow, Ian and McMahan, H. Brendan and Mironov, Ilya and Talwar, Kunal and Zhang, Li},
title = {Deep Learning with Differential Privacy},
year = {2016},
isbn = {9781450341394},
publisher = {Association for Computing Machinery},
address = {New York, NY, USA},
url = {https://doi.org/10.1145/2976749.2978318},
doi = {10.1145/2976749.2978318},
booktitle = {Proceedings of the 2016 ACM SIGSAC Conference on Computer and Communications Security},
pages = {308–318},
numpages = {11},
location = {Vienna, Austria},
series = {CCS '16}
}

@article{10.1561/0400000042,
    author = {Dwork, Cynthia and Roth, Aaron},
    title = {The Algorithmic Foundations of Differential Privacy},
    journal = {Foundations and Trends in Theoretical Computer Science},
    volume = {9},
    number = {3-4},
    pages = {211-487},
    year = {2014},
    month = {08},
    issn = {1551-305X},
    doi = {10.1561/0400000042},
    url = {https://doi.org/10.1561/0400000042},
    eprint = {https://www.emerald.com/fttcs/article-pdf/9/3-4/211/11486280/0400000042en.pdf},
}

@InProceedings{Zhao_2023_WACV,
    author    = {Zhao, Bo and Bilen, Hakan},
    title     = {Dataset Condensation With Distribution Matching},
    booktitle = {Proceedings of the IEEE/CVF Winter Conference on Applications of Computer Vision (WACV)},
    month     = {January},
    year      = {2023},
    pages     = {6514-6523}
}

@ARTICLE{10612802,
  author={Chen, Yuanlu and Ziying Tan, Alysa and Feng, Siwei and Yu, Han and Deng, Tao and Zhao, Libang and Wu, Feng},
  journal={IEEE Internet of Things Journal}, 
  title={General Federated Class-Incremental Learning With Lightweight Generative Replay}, 
  year={2024},
  volume={11},
  number={20},
  pages={33927-33939},
  doi={10.1109/JIOT.2024.3434600}}

@article{McCloskey1989CatastrophicII,
  title={Catastrophic Interference in Connectionist Networks: The Sequential Learning Problem},
  author={Michael McCloskey and Neal J. Cohen},
  journal={Psychology of Learning and Motivation},
  year={1989},
  volume={24},
  pages={109-165},
  url={https://api.semanticscholar.org/CorpusID:61019113}
}

@article{10.1016/j.neucom.2025.130844,
author = {Hamedi, Parisa and Razavi-Far, Roozbeh and Hallaji, Ehsan},
title = {Federated continual learning: Concepts, challenges, and solutions},
year = {2025},
issue_date = {Oct 2025},
publisher = {Elsevier Science Publishers B. V.},
address = {NLD},
volume = {651},
number = {C},
issn = {0925-2312},
url = {https://doi.org/10.1016/j.neucom.2025.130844},
doi = {10.1016/j.neucom.2025.130844},
journal = {Neurocomput.},
month = oct,
numpages = {26}
}

@InProceedings{pmlr-v267-li25cq,
  title = 	 {{F}ed{SSI}: Rehearsal-Free Continual Federated Learning with Synergistic Synaptic Intelligence},
  author =       {Li, Yichen and Wang, Yuying and Wang, Haozhao and Qi, Yining and Xiao, Tianzhe and Li, Ruixuan},
  booktitle = 	 {Proceedings of the 42nd International Conference on Machine Learning},
  pages = 	 {36151--36169},
  year = 	 {2025},
  editor = 	 {Singh, Aarti and Fazel, Maryam and Hsu, Daniel and Lacoste-Julien, Simon and Berkenkamp, Felix and Maharaj, Tegan and Wagstaff, Kiri and Zhu, Jerry},
  volume = 	 {267},
  series = 	 {Proceedings of Machine Learning Research},
  month = 	 {13--19 Jul},
  publisher =    {PMLR},
  url = 	 {https://proceedings.mlr.press/v267/li25cq.html}
}

@inproceedings{ijcai2023p443,
  title     = {FedET: A Communication-Efficient Federated Class-Incremental Learning Framework Based on Enhanced Transformer},
  author    = {Liu, Chenghao and Qu, Xiaoyang and Wang, Jianzong and Xiao, Jing},
  booktitle = {Proceedings of the Thirty-Second International Joint Conference on
               Artificial Intelligence, {IJCAI-23}},
  publisher = {International Joint Conferences on Artificial Intelligence Organization},
  editor    = {Edith Elkind},
  pages     = {3984--3992},
  year      = {2023},
  month     = {8},
  note      = {Main Track},
  doi       = {10.24963/ijcai.2023/443},
  url       = {https://doi.org/10.24963/ijcai.2023/443},
}

@InProceedings{pmlr-v235-piao24a,
  title = 	 {Federated Continual Learning via Prompt-based Dual Knowledge Transfer},
  author =       {Piao, Hongming and Wu, Yichen and Wu, Dapeng and Wei, Ying},
  booktitle = 	 {Proceedings of the 41st International Conference on Machine Learning},
  pages = 	 {40725--40739},
  year = 	 {2024},
  editor = 	 {Salakhutdinov, Ruslan and Kolter, Zico and Heller, Katherine and Weller, Adrian and Oliver, Nuria and Scarlett, Jonathan and Berkenkamp, Felix},
  volume = 	 {235},
  series = 	 {Proceedings of Machine Learning Research},
  month = 	 {21--27 Jul},
  publisher =    {PMLR},
  url = 	 {https://proceedings.mlr.press/v235/piao24a.html}
}

@INPROCEEDINGS{10184531,
  author={Luopan, Yaxin and Han, Rui and Zhang, Qinglong and Liu, Chi Harold and Wang, Guoren and Chen, Lydia Y.},
  booktitle={2023 IEEE 39th International Conference on Data Engineering (ICDE)}, 
  title={FedKNOW: Federated Continual Learning with Signature Task Knowledge Integration at Edge}, 
  year={2023},
  volume={},
  number={},
  pages={341-354},
  doi={10.1109/ICDE55515.2023.00033}}

@inproceedings{10.1007/978-3-031-73650-6_9,
author = {Guo, Haiyang and Zhu, Fei and Liu, Wenzhuo and Zhang, Xu-Yao and Liu, Cheng-Lin},
title = {PILoRA: Prototype Guided Incremental LoRA for Federated Class-Incremental Learning},
year = {2024},
isbn = {978-3-031-73649-0},
publisher = {Springer-Verlag},
address = {Berlin, Heidelberg},
url = {https://doi.org/10.1007/978-3-031-73650-6_9},
doi = {10.1007/978-3-031-73650-6_9},
booktitle = {Computer Vision -- ECCV 2024: 18th European Conference, Milan, Italy, September 29--October 4, 2024, Proceedings, Part LXV},
pages = {141--159},
numpages = {19},
location = {Milan, Italy}
}

@inproceedings{
qi2023better,
title={Better Generative Replay for Continual Federated Learning},
author={Daiqing Qi and Handong Zhao and Sheng Li},
booktitle={The Eleventh International Conference on Learning Representations },
year={2023},
url={https://openreview.net/forum?id=cRxYWKiTan}
}

@inproceedings{10.1145/1559845.1559850,
author = {McSherry, Frank D.},
title = {Privacy integrated queries: an extensible platform for privacy-preserving data analysis},
year = {2009},
isbn = {9781605585512},
publisher = {Association for Computing Machinery},
address = {New York, NY, USA},
url = {https://doi.org/10.1145/1559845.1559850},
doi = {10.1145/1559845.1559850},
booktitle = {Proceedings of the 2009 ACM SIGMOD International Conference on Management of Data},
pages = {19–30},
numpages = {12},
location = {Providence, Rhode Island, USA},
series = {SIGMOD '09}
}

@article{10.1109/TPAMI.2024.3367329,
author = {Wang, Liyuan and Zhang, Xingxing and Su, Hang and Zhu, Jun},
title = {A Comprehensive Survey of Continual Learning: Theory, Method and Application},
year = {2024},
issue_date = {Aug. 2024},
publisher = {IEEE Computer Society},
address = {USA},
volume = {46},
number = {8},
issn = {0162-8828},
url = {https://doi.org/10.1109/TPAMI.2024.3367329},
doi = {10.1109/TPAMI.2024.3367329},
journal = {IEEE Trans. Pattern Anal. Mach. Intell.},
month = aug,
pages = {5362–5383},
numpages = {22}
}

@INPROCEEDINGS {8049725,
author = { Mironov, Ilya },
booktitle = { 2017 IEEE 30th Computer Security Foundations Symposium (CSF) },
title = {{ Rényi Differential Privacy }},
year = {2017},
volume = {},
ISSN = {2374-8303},
pages = {263-275},
doi = {10.1109/CSF.2017.11},
url = {https://doi.ieeecomputersociety.org/10.1109/CSF.2017.11},
publisher = {IEEE Computer Society},
address = {Los Alamitos, CA, USA},
month =Aug}

@misc{mironov2019renyidifferentialprivacysampled,
      title={R\'enyi Differential Privacy of the Sampled Gaussian Mechanism}, 
      author={Ilya Mironov and Kunal Talwar and Li Zhang},
      year={2019},
      eprint={1908.10530},
      archivePrefix={arXiv},
      primaryClass={cs.LG},
      url={https://arxiv.org/abs/1908.10530}, 
}

@InProceedings{pmlr-v108-balle20a,
  title = 	 {Hypothesis Testing Interpretations and Renyi Differential Privacy},
  author =       {Balle, Borja and Barthe, Gilles and Gaboardi, Marco and Hsu, Justin and Sato, Tetsuya},
  booktitle = 	 {Proceedings of the Twenty Third International Conference on Artificial Intelligence and Statistics},
  pages = 	 {2496--2506},
  year = 	 {2020},
  editor = 	 {Chiappa, Silvia and Calandra, Roberto},
  volume = 	 {108},
  series = 	 {Proceedings of Machine Learning Research},
  month = 	 {26--28 Aug},
  publisher =    {PMLR},
  url = 	 {https://proceedings.mlr.press/v108/balle20a.html}
}

@ARTICLE{726791,
  author={Lecun, Y. and Bottou, L. and Bengio, Y. and Haffner, P.},
  journal={Proceedings of the IEEE}, 
  title={Gradient-based learning applied to document recognition}, 
  year={1998},
  volume={86},
  number={11},
  pages={2278-2324},
  doi={10.1109/5.726791}}

@inproceedings{NEURIPS2023_d160ea01,
 author = {Babakniya, Sara and Fabian, Zalan and He, Chaoyang and Soltanolkotabi, Mahdi and Avestimehr, Salman},
 booktitle = {Advances in Neural Information Processing Systems},
 doi = {10.52202/075280-2899},
 editor = {A. Oh and T. Naumann and A. Globerson and K. Saenko and M. Hardt and S. Levine},
 pages = {66408--66425},
 publisher = {Curran Associates, Inc.},
 title = {A Data-Free Approach to Mitigate Catastrophic Forgetting in Federated Class Incremental Learning for Vision Tasks},
 url = {https://proceedings.neurips.cc/paper_files/paper/2023/file/d160ea01902c33e30660851dfbac5980-Paper-Conference.pdf},
 volume = {36},
 year = {2023}
}

@inproceedings{lander,
title={Text-enhanced data-free approach for federated class-incremental learning},
author={Tran, Minh-Tuan and Le, Trung and Le, Xuan-May and Harandi, Mehrtash and Phung, Dinh},
booktitle={Proceedings of the IEEE/CVF Conference on Computer Vision and Pattern Recognition},
pages={23870--23880},
year={2024}
}
}

\clearpage
\appendix
\begin{center}
{\Large\bf Appendix}
\end{center}
\medskip

\startcontents[appendix]
\noindent{\large\bf Contents}
\medskip

\printcontents[appendix]{}{1}{\setcounter{tocdepth}{2}}
\clearpage

\section{Detailed Experimental Setup and Hyperparameters}
\label{app:hp}

We provide the detailed experimental settings, including the datasets and task
splits, the federated protocol, and the hyperparameters of ReSCENE and the
baselines.

\noindent\textbf{Datasets.}\quad
CIFAR-10 is split into five two-class tasks, CIFAR-100 into ten ten-class tasks,
and TinyImageNet into ten 20-class tasks. For these class-incremental datasets,
task-wise label heterogeneity follows a Dirichlet distribution with
$\beta\in\{0.1,0.5,1.0\}$. Office-31 evaluates domain-incremental learning over
the Amazon, DSLR, and Webcam sequence with a shared 31-class label space. As
specified in Sec.~\ref{sec:domain_plan}, it uses $N{=}10$, $K{=}4$, and
$\beta{=}1.0$, following the released code of Re-Fed~\citep{Li_2024_CVPR}. 

\noindent\textbf{Common protocol.}\quad
All methods use the federated protocol of Sec.~\ref{sec:setup}, with $N{=}20$
clients, $K{=}10$ participants per round, and $R{=}20$ communication rounds per
task. FedAvg, FedEWC, FedLwF, TARGET, MFCL, GLFC,
FedCBDR, and Re-Fed use $20$ local model-training epochs per participating
client and round, following Re-Fed~\citep{Li_2024_CVPR}. ReSCENE instead follows the client-side condensation schedule
specified below. The default model is ResNet-18 with an incrementally
expanded classifier head. CIFAR experiments use the CIFAR variant, whereas
TinyImageNet uses the $64\times64$ variant with max pooling. Client optimization
uses SGD with learning rate $0.01$, momentum $0.9$, weight decay
$5\times10^{-4}$, and batch size $128$. Augmentation consists of random crop and
horizontal flip, with padding $4$ for CIFAR and $8$ for TinyImageNet.

\noindent\textbf{ReSCENE.}\quad
Distribution matching uses IPC $10$, surrogates initialized from sampled real
images of each class as in FedDM, $25$ condensation steps, synthetic-image
learning rate $1.0$, and projected Gaussian parameter perturbation with radius
$\rho{=}5$. The server trains for $100$ epochs per round using SGD with initial
learning rate $0.01$, cosine annealing, momentum $0.9$, and weight decay
$5\times10^{-4}$. The per-class buffer budget is $B{=}1000$, the balanced
sampling exponent is $\alpha{=}0.5$, and the logit-adjustment strength is
$\tau{=}1.0$. Temporal herding uses the temporal fraction
$\lambda{=}0.75$. When differential privacy is required, the surrogates are
initialized from random noise and optimized by Algorithm~\ref{alg:dp}
(Appendix~\ref{app:privacy}).

\noindent\textbf{FedAvg~\citep{pmlr-v54-mcmahan17a}.}\quad
No additional hyperparameters beyond the common protocol.

\noindent\textbf{FedEWC~\citep{doi:10.1073/pnas.1611835114}.}\quad
EWC penalty weight $\lambda_{\mathrm{EWC}}{=}1000$. Fisher information is recomputed
on each participating client's local data after every round of local training,
and Table~\ref{tab:resource} includes this extra pass.

\noindent\textbf{FedLwF~\citep{10.1109/TPAMI.2017.2773081}.}\quad
Knowledge distillation temperature $T_{\mathrm{KD}}{=}2$, KD loss weight
$\lambda_{\mathrm{KD}}{=}3$.

\noindent\textbf{TARGET~\citep{Zhang_2023_ICCV}.}\quad
Server-side synthesis with a DeepInversion generator (latent dimension $256$,
base width $64$, batch size $256$) for $100$ synthesis rounds, with a $20$-round
warm-up and $400$ student distillation steps per round. Per synthesis round,
CIFAR-10/100 use $10$ generator steps, generator learning rate $0.002$, loss
weights $(\lambda_{\mathrm{BN}},\lambda_{\mathrm{OH}},\lambda_{\mathrm{adv}}){=}(10,0.5,1)$,
and student temperature $20$, and TinyImageNet uses TARGET's dataset-specific
setting of $50$ generator steps, learning rate $0.0002$, loss weights
$(0.1,0.1,1)$, and temperature $5$. Client-side distillation on the synthesized
data uses KD weight $25$ and temperature $2$.

\noindent\textbf{MFCL~\citep{NEURIPS2023_d160ea01}.}\quad
At the end of each task, the server trains a generator against the global model
without data for $5{,}000$ iterations with Adam (learning rate $10^{-3}$, batch
size $128$, noise dimension $200$), using the content, diversity
($w_{\mathrm{div}}{=}1$), batch-statistics ($w_{\mathrm{BN}}{=}75$), and image-prior
($w_{\mathrm{pr}}{=}10^{-3}$) losses of the original paper. The generator follows
the original architecture for $32\times32$ images and its $64\times64$ variant for
TinyImageNet. In each local step, a client draws a synthetic batch of the same
size as its real batch and trains with cross-entropy on the new classes,
fine-tuning of the classifier head on real and synthetic samples
($w_{\mathrm{FT}}{=}1$), and feature distillation from the previous global model
($w_{\mathrm{KD}}{=}0.1$).

\noindent\textbf{GLFC~\citep{Dong_2022_CVPR}.}\quad
Loss weights $w_{\mathrm{GC}}{=}w_{\mathrm{RD}}{=}0.5$ for the class-aware gradient
compensation and relation distillation losses. Exemplar memory
budget $20$ images per class (herding selection). Prototype perturbation takes $50$
SGD iterations with one tenth of the local learning rate and weight decay $10^{-5}$. Server-side
gradient inversion (LBFGS) runs $250$ iterations with learning rate $0.1$. A LeNet-based encoding
model is used for prototype gradient communication.

\noindent\textbf{FedCBDR~\citep{NEURIPS2025_d611d06e}.}\quad
Task-aware temperature scaling uses $\tau_{\mathrm{old}}{=}0.8$,
$\tau_{\mathrm{new}}{=}1.2$, $w_{\mathrm{old}}{=}1.2$, and
$w_{\mathrm{new}}{=}0.8$. Exemplar budget $150$ images per class.
Coreset selection via SVD leverage-score sampling across clients.

\noindent\textbf{Re-Fed~\citep{Li_2024_CVPR}.}\quad
Local storage budget $M{=}1000$ for CIFAR-10/100, $M{=}2000$ for TinyImageNet,
and $M{=}300$ for Office-31, covering the current-task data and the cached
previous samples together. The personalized informative model is trained for
$40$ epochs with learning rate $0.01$, batch size $64$, and proximal coefficient
$(1-\lambda_{\mathrm{PIM}})/(2\lambda_{\mathrm{PIM}})$ with $\lambda_{\mathrm{PIM}}{=}0.5$. Previous samples are
ranked by gradient-norm importance with the early-emphasis $1/p$ weighting.

\noindent\textbf{FedDM~\citep{Xiong_2023_CVPR}.}\quad
Distribution matching with $\mathrm{IPC}{=}10$, $25$ condensation steps,
synthetic-image learning rate $1.0$, and Gaussian model perturbation with radius
$\rho{=}5$, matching feature and logit means. The server trains for $100$ epochs
per round on the surrogates received in that round only, with neither balanced
sampling nor logit adjustment.

\noindent\textbf{FedAF~\citep{Wang_2024_CVPR}.}\quad
Same condensation budget and server schedule as FedDM. Following the public
implementation, collaborative data condensation weight
$\lambda_{\mathrm{cdc}}{=}10^{-4}$, local-global knowledge matching weight
$\lambda_{\mathrm{lgkm}}{=}0.01$, soft-label temperature $2$, and model mixing
coefficient $0.9$. The perturbation noise is drawn at the scale of each weight
tensor, since unit-variance noise drives the ResNet-18 features to overflow.

\section{Extended Related Work}
\label{app:related}

This appendix expands the discussion of Section~\ref{sec:related}.

Federated continual learning methods are conventionally grouped into regularization-, distillation-,
architecture-, and replay-based families~\citep{10.1109/TPAMI.2024.3367329,10.1016/j.neucom.2025.130844}. We summarize the representative methods of each
family and the limitation the family shares.

\noindent\textbf{Regularization-based methods.}\quad Regularization-based
methods preserve earlier-task knowledge by restraining the change of parameters
that contributed most to previous tasks while learning a new task.
EWC~\citep{doi:10.1073/pnas.1611835114} penalizes the change of each weight in
proportion to its Fisher importance for the earlier tasks. 
FedSSI~\citep{pmlr-v267-li25cq}
measures that importance on each client along the trajectory of a surrogate
model that is pulled toward the global model, so that the importance on a
client reflects the global model as well as its own data.
FedKNOW~\citep{10184531} stores the most important weights of each finished
task and adjusts the gradient of a new task so that it does not conflict with
the gradients of the past tasks least similar to it. 
The shared limitation is that the constraints accumulate with the task sequence.
Since every finished task restricts the weight change individually, the longer
the sequence, the less the model is allowed to change. Therefore, new knowledge 
is hindered more and more toward the end of the sequence.

\noindent\textbf{Distillation-based methods.}\quad 
Distillation-based methods set the earlier-task model as a teacher and force the 
current model to reproduce its outputs while learning the new task. 
LwF~\citep{10.1109/TPAMI.2017.2773081} records the outputs of the previous model
on the new-task images and penalizes deviation from them during new-task training.
CFeD~\citep{ijcai2022p303} applies distillation on both the clients and the server.
Each client learns the new task under the previous-task model as its teacher, 
and the server distills the received local models into the averaged global model. 
FedET~\citep{ijcai2023p443} learns small adapters on top of a frozen pre-trained
transformer, treating the adapter of the old tasks as the teacher for the one trained
on the new task. 
The common drawback of these methods is that the teacher is a stale model
trained on the earlier tasks alone, so forcing the current model to reproduce
its outputs hinders learning of the new task. 
Furthermore, in class-incremental learning the class distribution completely changes
from one task to the next, so the new task and the teacher model come from 
entirely different distributions, which makes it much harder to transfer knowledge.

\noindent\textbf{Architecture-based methods.}\quad
Architecture-based methods assign each task a dedicated subset of 
parameters or a separate module, so that the knowledge of a task 
stays untouched when later tasks are learned. 
FedWeIT~\citep{pmlr-v139-yoon21b} splits each client's model into a shared base
and sparse task-specific parameters, and exchanges the task-specific parameters
among clients. 
Powder~\citep{pmlr-v235-piao24a} attaches a separate set of prompts to 
a frozen pre-trained backbone for each task, and the server pools these 
prompts and shares the related ones across clients.
PILoRA~\citep{10.1007/978-3-031-73650-6_9} adds, for each task, a separate
low-rank adapter to a frozen pre-trained backbone, trained so that it does not
overlap with the adapters of earlier tasks.
In the federated setting, however, task-specific parameters are hard to
coordinate across clients whose task sequences and classes differ. Moreover, as
the parameters set aside for tasks accumulate, the capacity left for a new task
shrinks, as it does for regularization-based methods. Another drawback of Powder
and PILoRA is that they rest on a pre-trained backbone that every client must hold.

\noindent\textbf{Replay-based methods.}\quad Replay-based methods rehearse past
data alongside the new data. We organize them by where the model is trained on
the replay data, on the clients or at the server.

\textit{Exemplar-based replay.} Exemplar-based methods store a
subset of samples, raw or condensed, on the clients and mix them into local
training.
iCaRL~\citep{Rebuffi_2017_CVPR} selects exemplars by herding, and
GLFC~\citep{Dong_2022_CVPR} keeps a per-class exemplar memory on each
client and corrects the gradients of old classes during local
training. 
ECoral~\citep{sun2025exemplarcondensedfederatedclassincrementallearning} condenses
the training data of each client into the exemplars it stores, so that each
stored sample carries more information. 
Re-Fed~\citep{Li_2024_CVPR} caches the previous-task samples that a
personalized model scores as most informative. 
FedCBDR~\citep{NEURIPS2025_d611d06e}
lets the server choose the exemplars from obfuscated client features so that the
replay is class-balanced across clients, while the selected samples stay on the
clients that own them.
What limits these methods is that they burden the resource-poor clients.
Data of earlier tasks have to stay in the client's scarce storage, and as
tasks accumulate, this storage either grows or spreads a fixed budget more thinly
over more classes.
Furthermore, each client trains on the exemplars
alongside the current task, which adds computation to every local update.
In addition, each exemplar memory is built from the data of its own client, so
it holds only the classes that the client has observed, and under non-IID data
the replay memory is split across the clients in the same way as the data.
Each client then rehearses a different subset of the earlier classes in its local
training before aggregation, so the replay remains subject to client drift.

\textit{Generative replay.} Generative methods, inspired
by deep generative replay in
centralized continual learning~\citep{NIPS2017_0efbe980} and by model inversion,
synthesize past-class data
with a generator in place of stored samples, and several of them build these
data at the server.
TARGET~\citep{Zhang_2023_ICCV}
inverts the previous global model into a generator on the server, sends the
generated images to the clients, and has each client train on them alongside its
new-task data while distilling the previous global model.
MFCL~\citep{NEURIPS2023_d160ea01} trains the generator on the server at the end
of each task from the global model without any data and sends it to the clients
together with the previous global model. Each client then draws synthetic
samples at every step, fine-tunes the classifier head on real and synthetic
samples, and distills the features of the previous model.
LANDER~\citep{lander} conditions the server-side generator on label-text
embeddings of a pre-trained text encoder, sends the synthesized images and the
previous global model to the clients, and has each client distill the previous
model on the synthesized images.
GenFCIL~\citep{10612802} trains a lightweight generator on
the server from the classifiers and feature representations that clients upload, and sends
generated features of the old classes back to the clients, which train on them
together with their new data. 
In these methods, the server builds the replay data, the global model is
obtained by aggregation, and the replay training itself remains on the clients,
which hold the synthesized data or the generator, in most of them together with
an earlier model, and spend additional computation on them in every local update.
Since the clients replay these data in local training before aggregation, the
replay also remains subject to client drift under non-IID data, as in
exemplar-based replay.
The synthesized data are also a poor
substitute for the real data of the earlier classes, since a server-side
generator relies on the global model and auxiliary signals in place of real data.

\textit{Replay training at the server.}
FedCIL~\citep{qi2023better} equips every client and the server with a
conditional generative adversarial network, 
so that each client replays its earlier classes from its own
generator, and the server consolidates the aggregated global model in every round
on data drawn from the clients' generators. 
The server thus takes part in the replay training, but every client still trains
a generator and replays from it before aggregation, and FedCIL itself reports that
training these client-side generators is unstable under non-IID data, which its
server-side consolidation is designed to counter.
In ReSCENE, the replay training takes place only at the server. The clients
upload surrogates condensed from their real data, keep no replay data or earlier
model, and never train the global model, so the replay involves no local
training and is free of client drift.

\section{Theoretical Analysis of Temporal Herding}
\label{app:temporal_target}

This section provides the full theoretical analysis behind
Sec.~\ref{sec:theory_short}. Temporal herding uses the recent surrogates to
define a target mean but selects the retained surrogates from the entire task
pool. After fixing the notation (Sec.~\ref{app:theory_setup}), we first show why
the recent surrogates make a suitable target, since they carry the least
staleness under the task-end model (Proposition~\ref{lem:terminal_stale},
Sec.~\ref{app:theory_staleness}). We then bound the error of the finite buffer
built toward this target (Lemma~\ref{thm:greedy_target_bound},
Sec.~\ref{app:theory_bound}). Finally, we establish when this buffer represents
the real task-end class mean more closely than two alternatives, full
accumulation of the whole pool without any selection and full-pool herding toward
the mean of the whole pool with $\lambda{=}1$
(Theorem~\ref{thm:greedy_full_improvement} and Corollary~\ref{cor:fullpool},
Sec.~\ref{app:theory_full}). Remark~\ref{rem:temporal_herding_meaning} states the scope of
the analysis.

\subsection{Setup}
\label{app:theory_setup}

We fix task $t$ and class $c$. We use $R$ communication rounds per task,
$r\in\{1,\ldots,R\}$ to index a round within task $t$, and $B$ for the number
of class-$c$ surrogates retained at the task boundary. We define $P_r$ as the
index set of class-$c$ surrogates uploaded in round $r$, with surrogate $s_i$ for
each $i\in P_r$, and $\mathcal{I}=\{r:|P_r|>0\}$ as the set of rounds containing
at least one such surrogate. We then define
\[
 P=\bigcup_{r\in\mathcal{I}}P_r,
 \qquad
 W=\bigcup_{r\in\mathcal{I}_W}P_r
 \subseteq P
\]
as the candidate index set and the recent target index set, respectively. The
parameter $\lambda\in(0,1]$ specifies the fraction of recent communication
rounds used to form the target, and $w=\lceil\lambda R\rceil$ is the resulting
number of rounds. We define
$\mathcal{I}_W=\{r\in\mathcal{I}:r\ge R-w+1\}$. If this set is empty, we use
$\mathcal{I}_W=\{\max\mathcal{I}\}$. We denote the global model after the
server update in round $r$ by $\theta_t^{(r)}$. Hence
$\theta_t^{(R)}$ is the task-end model. We define the task-end feature vector
of each candidate and, for any nonempty index set $A\subseteq P$, its feature
mean as
\begin{equation}
 \vz_i
 =
 \frac{\phi_{\theta_t^{(R)}}(s_i)}
      {\|\phi_{\theta_t^{(R)}}(s_i)\|},
 \qquad
 \vm_A=\frac1{|A|}\sum_{i\in A}\vz_i .
\label{eq:set_feature_mean}
\end{equation}
Here $\vz_i$ is the unit-normalized task-end feature vector of surrogate $s_i$,
and $i$ indexes a candidate surrogate. In particular, $\vm_W$ is the
mean of the recent candidates' normalized feature vectors and serves as the
herding target.
As $\lambda$ is a fraction of rounds and $|P_r|$ varies across rounds, the
sample fraction $|W|/|P|$ differs from $\lambda$ in general.
Temporal herding first uses the recent index set $W$ to define the target mean
$\vm_W$, and then keeps every surrogate indexed by the full task pool $P$ eligible
and constructs a buffer by greedy selection toward this target. We use $\ell\in\{1,\ldots,B\}$ for this selection step and $H_\ell$ for
the indices selected through step $\ell$, with $H_0=\varnothing$. At step
$\ell$, $i\in P\setminus H_{\ell-1}$ is the index of a surrogate that has not
yet been selected, $s_i$ is the corresponding surrogate, and
$\vm_{H_{\ell-1}\cup\{i\}}$ is the feature mean that would result from adding
$s_i$. The symbol $i_\ell$ denotes the index of the surrogate actually selected
at step $\ell$. Temporal herding selects one surrogate without replacement at
each step according to
\begin{equation}
 i_\ell\in\argmin_{i\in P\setminus H_{\ell-1}}
 \|\vm_W-\vm_{H_{\ell-1}\cup\{i\}}\|^2,
 \qquad
 H_\ell=H_{\ell-1}\cup\{i_\ell\}.
\label{eq:temporal_herding}
\end{equation}
This is Eq.~\eqref{eq:herding} written with indices, where $s_{i_\ell}=p_\ell$ and
$H_B$ indexes the buffer $\B_{t,c}$.

\subsection{Recent Surrogates Reduce Task-End Error}
\label{app:theory_staleness}

We first show why temporal herding takes the recent surrogates as its target.
A surrogate produced in round $r$ is matched to the real class under the model
available in that round, but temporal herding evaluates all candidates under
the task-end model. We quantify the mismatch introduced by this model
change. For a model $\theta$, let
\[
 \vm_r(\theta)=\frac1{|P_r|}\sum_{i\in P_r}\frac{\phi_\theta(s_i)}{\|\phi_\theta(s_i)\|},
 \qquad
 \vmu(\theta)=\mathbb E_{x\sim\mathcal D_c}
 \left[\frac{\phi_\theta(x)}{\|\phi_\theta(x)\|}\right]
\]
be the mean normalized feature of the round-$r$ surrogates and of the real
class-$c$ data distribution $\mathcal D_c$, respectively. Under the task-end
model, $\vm_r(\theta_t^{(R)})=\frac1{|P_r|}\sum_{i\in P_r}\vz_i$, and
$\vmu=\vmu(\theta_t^{(R)})$ is the real class mean that the buffer should
represent. The matching error of the round-$r$ surrogates is measured under the
global model $\theta_t^{(r-1)}$ used in that round,
\[
 \delta_r
 =
 \|\vm_r(\theta_t^{(r-1)})-\vmu(\theta_t^{(r-1)})\|.
\]

To transfer this matching error to the task-end feature space, we require that
normalized feature means change smoothly along the bounded within-task model
trajectory. This is a standard local smoothness condition, and it is needed only
over the parameters that the model visits within one task.

\begin{proposition}[Task-end surrogate error bound]
\label{lem:terminal_stale}
Suppose that for any two model states $\theta_t^{(a)}$ and $\theta_t^{(b)}$ on
the within-task trajectory, where $0\le a\le b\le R$, and each
$\vm\in\{\vmu\}\cup\{\vm_r:r\in\mathcal{I}\}$,
\begin{equation}
 \|\vm(\theta_t^{(b)})-\vm(\theta_t^{(a)})\|
 \le L_\phi\|\theta_t^{(b)}-\theta_t^{(a)}\|.
\label{eq:temporal_lip}
\end{equation}
Then the task-end error of the round-$r$ surrogates satisfies
\begin{equation}
 \|\vm_r(\theta_t^{(R)})-\vmu\|
 \le
 \delta_r+2L_\phi
 \|\theta_t^{(R)}-\theta_t^{(r-1)}\|.
\label{eq:terminal_stale}
\end{equation}
\end{proposition}

\noindent\textit{Proof.}
Applying the triangle inequality to the left-hand side of
Eq.~\eqref{eq:terminal_stale} through the model $\theta_t^{(r-1)}$ and then
Eq.~\eqref{eq:temporal_lip} gives
\begin{align*}
&\|\vm_r(\theta_t^{(R)})-\vmu(\theta_t^{(R)})\| \nonumber\\
&\le
 \|\vm_r(\theta_t^{(R)})-\vm_r(\theta_t^{(r-1)})\|
 +\|\vm_r(\theta_t^{(r-1)})-\vmu(\theta_t^{(r-1)})\| \nonumber\\
&\quad
 +\|\vmu(\theta_t^{(r-1)})-\vmu(\theta_t^{(R)})\| \\
&\le
 L_\phi\|\theta_t^{(R)}-\theta_t^{(r-1)}\|
 +\delta_r
 +L_\phi\|\theta_t^{(R)}-\theta_t^{(r-1)}\| \\
&=
 \delta_r+2L_\phi
 \|\theta_t^{(R)}-\theta_t^{(r-1)}\|.
\end{align*}

\begin{remark}[Recent targets reduce the staleness component]
\label{rem:task_end_staleness}
Proposition~\ref{lem:terminal_stale} bounds the task-end error of a round's
surrogates by the sum of their matching error $\delta_r$ and a staleness term
determined by how far the model is subsequently trained. Recent surrogates are associated
with model states closer to $\theta_t^{(R)}$ and therefore reduce this
staleness component, which motivates defining the herding target with recent
rounds.
\end{remark}

\subsection{Temporal Herding Approximates the Target within a Bounded Error}
\label{app:theory_bound}

Proposition~\ref{lem:terminal_stale} motivates the recent target $\vm_W$, but ReSCENE retains only $B$
surrogates per class. We now show that the implemented greedy selection keeps
the finite buffer close to this target. Using the set-mean definition in
Eq.~\eqref{eq:set_feature_mean}, $\vm_P$ and $\vm_{H_B}$ denote the feature means
of the candidate pool and the final herded buffer, respectively. Both are
computed using the task-end model $\theta_t^{(R)}$. We also define
\[
 \Delta_W=\max_{i\in P}\|\vz_i-\vm_W\|,
\]
which is the maximum feature-space distance from a candidate to the recent
target $\vm_W$. As $\|\vz_i\|=1$ and $\vm_W$ lies in the unit ball,
$\Delta_W\le2$.

\medskip
\noindent\textbf{Tracking the greedy selection error.}\quad
To track how closely the selected candidates match $\vm_W$, we define the
residual before selection step $\ell$ as
\[
 \vr_{\ell-1}=(\ell-1)\vm_W-\sum_{j\in H_{\ell-1}}\vz_j,
 \qquad \vr_0=0.
\]
For $\ell>1$, the set-mean definition in Eq.~\eqref{eq:set_feature_mean} gives
\[
 \vr_{\ell-1}
 =(\ell-1)(\vm_W-\vm_{H_{\ell-1}}),
\]
so $\vr_{\ell-1}$ is the feature-mean error after $\ell-1$ selections, scaled
by the number of selected candidates. The base case $\vr_0=0$ follows from
$H_0=\varnothing$ before the first selection.

For an unselected candidate index $i\in P\setminus H_{\ell-1}$, the
hypothetical mean error after adding surrogate $s_i$ is
\begin{align*}
 \vm_W-\vm_{H_{\ell-1}\cup\{i\}}
 &=\frac1\ell
 \left[\vr_{\ell-1}-(\vz_i-\vm_W)\right],\nonumber\\
 \|\vm_W-\vm_{H_{\ell-1}\cup\{i\}}\|^2
 &=\frac1{\ell^2}
 \|\vr_{\ell-1}-(\vz_i-\vm_W)\|^2.
\end{align*}
The positive factor $1/\ell^2$ is shared by all unselected candidate indices
$i$, so the same surrogate minimizes both objectives and
Eq.~\eqref{eq:temporal_herding} is equivalent to
\begin{equation}
 i_\ell\in
 \argmin_{i\in P\setminus H_{\ell-1}}
 \|\vr_{\ell-1}-(\vz_i-\vm_W)\|^2.
\label{eq:residual_argmin}
\end{equation}
After temporal herding selects surrogate $s_{i_\ell}$, whose index is
$i_\ell$, the actual residual is updated as
\begin{equation}
 \vr_\ell
 =\vr_{\ell-1}-(\vz_{i_\ell}-\vm_W)
 =\ell(\vm_W-\vm_{H_\ell}).
\label{eq:residual_update}
\end{equation}
The residual thus gives an equivalent form of the greedy mean-matching rule, in
which each step picks the candidate that minimizes the next residual norm.

\medskip
\noindent\textbf{Condition for a finite-buffer guarantee.}\quad
This equivalence identifies what the greedy rule minimizes. To bound the residual
over all $B$ selections, the remaining pool must also retain a candidate that can
move the selected mean toward the target.

\begin{assumption}[Directional coverage]
\label{asm:directional_coverage}
For every $\ell=1,\ldots,B$, the remaining candidate-index set satisfies
\begin{equation}
 \max_{i\in P\setminus H_{\ell-1}}
 \langle \vr_{\ell-1},\vz_i-\vm_W\rangle\ge0.
\label{eq:directional_coverage}
\end{equation}
\end{assumption}

This condition is natural for temporal herding. The target $\vm_W$ is the mean
of the recent index set $W\subseteq P$, and the maximum of a set along any
direction is at least its mean, so for every direction some point of $W$ lies
at or beyond the target. Such a point, once added, moves the selected mean
toward the target, and the condition holds at step $\ell$ whenever such a point
of $W$ along $\vr_{\ell-1}$ has not yet been selected.
For $\ell\ge2$, the residual $\vr_{\ell-1}$ is a positive multiple of
$\vm_W-\vm_{H_{\ell-1}}$, so this is the condition stated in
Sec.~\ref{sec:theory_short}, and it holds trivially for $\ell=1$ since $\vr_0=0$.

\begin{lemma}[Finite-budget bound for greedy temporal herding]
\label{thm:greedy_target_bound}
For $B\le|P|$, under Assumption~\ref{asm:directional_coverage}, the distinct,
no-replacement greedy rule in \eqref{eq:temporal_herding} satisfies
\begin{equation}
 \|\vm_{H_B}-\vm_W\|\le\frac{\Delta_W}{\sqrt B}\le\frac{2}{\sqrt B}.
\label{eq:greedy_target_bound}
\end{equation}
\end{lemma}

\noindent\textit{Proof.}
By Eq.~\eqref{eq:residual_argmin}, the implemented greedy rule selects the
surrogate whose index minimizes the new residual norm. By
Assumption~\ref{asm:directional_coverage}, there is an
unselected candidate index $i^\star$ with
$\langle \vr_{\ell-1},\vz_{i^\star}-\vm_W\rangle\ge0$. Greedy optimality therefore gives
\begin{align*}
 \|\vr_\ell\|^2
 &\le\|\vr_{\ell-1}-(\vz_{i^\star}-\vm_W)\|^2 \\
 &=\|\vr_{\ell-1}\|^2
 -2\langle\vr_{\ell-1},\vz_{i^\star}-\vm_W\rangle
 +\|\vz_{i^\star}-\vm_W\|^2 \\
 &\le\|\vr_{\ell-1}\|^2+\Delta_W^2.
\end{align*}
The last step uses the nonnegative inner product from the assumption and
$\|\vz_{i^\star}-\vm_W\|\le \Delta_W$. Induction from $\vr_0=0$ yields
$\|\vr_B\|^2\le B\Delta_W^2$. Finally, the residual definition and
$\sum_{i\in H_B}\vz_i=B\vm_{H_B}$ give
\[
 \vr_B=B(\vm_W-\vm_{H_B}).
\]
Hence $\|\vm_{H_B}-\vm_W\|=\|\vr_B\|/B\le \Delta_W/\sqrt B$, proving
\eqref{eq:greedy_target_bound}.

The approximation error of Lemma~\ref{thm:greedy_target_bound} shrinks as
$1/\sqrt B$ and is at most $2/\sqrt B$, so the finite buffer stays close to the
recent target that Proposition~\ref{lem:terminal_stale} motivates.

\subsection{When the Bounded Buffer Surpasses Full Accumulation}
\label{app:theory_full}

Lemma~\ref{thm:greedy_target_bound} shows that the buffer stays close to the
recent target. To compare it with full accumulation, we next describe how the
task-end error of the surrogates changes over rounds. Recall that $\vmu=\vmu(\theta_t^{(R)})$ is the real task-end class mean and that
$\vm_r(\theta_t^{(R)})=\frac1{n_r}\sum_{i\in P_r}\vz_i$, with $n_r=|P_r|$, is
the task-end mean of the surrogates produced in round $r$.
Proposition~\ref{lem:terminal_stale} bounds the task-end error of each round by a
staleness term that shrinks with $r$, which suggests that the magnitude of the
error decreases over rounds. Its direction is shared across rounds as well. Within a task the
model moves along a single trajectory from $\theta_t^{(0)}$ to $\theta_t^{(R)}$,
and a surrogate matched under an earlier state is displaced, under the task-end
model, by the change of the feature map along that trajectory, so the errors of
different rounds point in a common direction, with earlier rounds displaced
further. Assumption~\ref{asm:coherent_bias} states this structure, with a
deviation of bounded norm that covers client sampling and round-to-round
variation.

\begin{assumption}[Coherent stale bias]
\label{asm:coherent_bias}
There exist a unit vector $\vu$, nonnegative coefficients $a_r$ for
$r\in\mathcal{I}$, and deviation vectors $\vxi_r$ such that
\begin{equation}
 \vm_r(\theta_t^{(R)})-\vmu=a_r\vu+\vxi_r,
 \qquad
 \|\vxi_r\|\le\nu.
\label{eq:coherent_bias}
\end{equation}
\end{assumption}

Assumption~\ref{asm:coherent_bias} places little restriction on the errors
themselves, since the deviation $\vxi_r$ absorbs whatever part of the error is not
aligned with $\vu$, and any error vectors satisfy it with a large enough
$\nu$. What matters is how large the coherent part is relative to
this deviation, which is exactly what the condition of
Theorem~\ref{thm:greedy_full_improvement} compares.

The sample-weighted stale biases of the whole pool and the recent window are
\[
 A_P
 =
 \frac{\sum_{r\in\mathcal{I}}n_ra_r}{\sum_{r\in\mathcal{I}}n_r},
 \qquad
 A_W
 =
 \frac{\sum_{r\in\mathcal{I}_W}n_ra_r}
      {\sum_{r\in\mathcal{I}_W}n_r}.
\]
When $a_r$ decreases over the rounds, every excluded early coefficient is at
least every included recent coefficient, since $W$ is a suffix, and hence
$A_W\le A_P$. The recent target then removes a stale bias of magnitude
$A_P-A_W$ relative to full accumulation. It remains to determine when this removed
bias outweighs the deviation of the surrogates and the approximation error of
Lemma~\ref{thm:greedy_target_bound}.

\begin{theorem}[Greedy improvement over full accumulation]
\label{thm:greedy_full_improvement}
Under Assumptions~\ref{asm:directional_coverage} and
\ref{asm:coherent_bias}, the buffer selected by the greedy rule in
Eq.~\eqref{eq:temporal_herding} is strictly closer to $\vmu$ than full
accumulation whenever
\begin{equation}
 A_P-A_W>2\nu+\frac{\Delta_W}{\sqrt B}.
\label{eq:bias_margin}
\end{equation}
Specifically,
\begin{equation}
 \|\vm_{H_B}-\vmu\|<\|\vm_P-\vmu\|.
\label{eq:herding_beats_full}
\end{equation}
\end{theorem}

\noindent\textit{Proof.}
Averages of the deviations in \eqref{eq:coherent_bias} have norm at most
$\nu$.  The reverse triangle inequality therefore gives
$\|\vm_P-\vmu\|\ge A_P-\nu$, while the triangle inequality gives
$\|\vm_W-\vmu\|\le A_W+\nu$.  Lemma~\ref{thm:greedy_target_bound} and another
triangle inequality give
\[
 \|\vm_{H_B}-\vmu\|\le A_W+\nu+\frac{\Delta_W}{\sqrt B}.
\]
Under Condition~\eqref{eq:bias_margin}, these bounds combine as
\[
 \|\vm_{H_B}-\vmu\|
 \le A_W+\nu+\frac{\Delta_W}{\sqrt B}
 < A_P-\nu
 \le \|\vm_P-\vmu\|,
\]
which proves Eq.~\eqref{eq:herding_beats_full}.

\begin{corollary}[Improvement over full-pool herding]
\label{cor:fullpool}
Let $H_B^{P}$ be the buffer selected by the same greedy rule with the full-pool
target $\vm_P$ in place of $\vm_W$ (that is, $\lambda{=}1$), and let
$\Delta_P=\max_{i\in P}\|\vz_i-\vm_P\|$. Under
Assumptions~\ref{asm:directional_coverage} (for both targets) and
\ref{asm:coherent_bias},
\[
 A_P-A_W>2\nu+\frac{\Delta_W+\Delta_P}{\sqrt B}
 \quad\Longrightarrow\quad
 \|\vm_{H_B}-\vmu\|<\|\vm_{H_B^{P}}-\vmu\|.
\]
\end{corollary}

\noindent\textit{Proof.}
Lemma~\ref{thm:greedy_target_bound} applied to the full-pool target gives
$\|\vm_{H_B^{P}}-\vm_P\|\le\Delta_P/\sqrt B$, so by the reverse triangle
inequality and $\|\vm_P-\vmu\|\ge A_P-\nu$,
$\|\vm_{H_B^{P}}-\vmu\|\ge A_P-\nu-\Delta_P/\sqrt B$. The proof of
Theorem~\ref{thm:greedy_full_improvement} gives
$\|\vm_{H_B}-\vmu\|\le A_W+\nu+\Delta_W/\sqrt B$, and the stated margin makes
the upper bound smaller than the lower bound.

\medskip
\noindent\textbf{When the condition holds.}\quad
Condition~\eqref{eq:bias_margin} compares three quantities that are governed by
different parts of the setting. The deviation bound $\nu$ is set by how far
the surrogates of a round deviate from their stale bias, which comes from
client sampling and round-to-round variation, and the approximation error
$\Delta_W/\sqrt B$ is set by the budget, with $\Delta_W\le 2$ for unit-normalized
features. The removed bias $A_P-A_W$ depends on how far the model moves within
the task. As the rounds proceed the model keeps learning, the surrogates of
the early rounds are matched under models ever farther from $\theta_t^{(R)}$, and
the bound of Proposition~\ref{lem:terminal_stale} on their task-end error grows
relative to that of the late rounds, so the gap between the average bias of the
whole pool and that of the recent window widens. When the model keeps changing
until the last rounds of a task, the removed bias can exceed the other two terms,
and Theorem~\ref{thm:greedy_full_improvement} then states that a bounded buffer
aimed at the recent rounds represents the class more closely than the whole pool.
When the model converges early in the task, the bias of the rounds flattens, the
removed bias shrinks, and the theorem gives no advantage over full accumulation.

\begin{remark}[Scope of the analysis]
\label{rem:temporal_herding_meaning}
Under Condition~\eqref{eq:bias_margin}, Equation~\eqref{eq:herding_beats_full}
shows that a bounded buffer aimed at the recent rounds represents a completed task
more closely than retaining every surrogate. Full accumulation gives early, stale generations the same influence
on its mean as later generations. Temporal herding lets the recent generations
define the target while keeping every surrogate eligible for selection, so
useful earlier samples still enter the buffer while the recent generations set
its representation. The analysis concerns the class mean of normalized features under the task-end
model, a counterpart of the mean embedding that distribution matching aligns, and it does not bound the
accuracy after later tasks change the model, which Sec.~\ref{sec:herding_main}
examines empirically.
\end{remark}

\FloatBarrier

\section{Additional Experiments}
\label{app:experiments}

We report additional experiments of ReSCENE. We report the per-task accuracies
behind Table~\ref{tab:main} (Sec.~\ref{app:forgetting}), evaluate ReSCENE in the
domain-incremental setting (Sec.~\ref{sec:domain_plan}), and compare it with the
replay baselines under the same replay budget for a fair comparison
(Sec.~\ref{app:cost_memory}). We also sweep the two hyperparameters of temporal
herding (Sec.~\ref{sec:policy_ablation}) and ablate the class-imbalance
corrections of the server objective (Sec.~\ref{app:balance_logit_ablation}).

\subsection{Detailed Per-Task Results}
\label{app:forgetting}

AA and AIA in Table~\ref{tab:main} average the accuracy over tasks, so two
methods with the same average can forget in very different ways. To show how
each method retains the earlier tasks, we report the final accuracy of every task
after the complete sequence, for every setting of Tables~\ref{tab:main}
and~\ref{tab:safl}, and for FedDM and FedAF also at $\beta{=}0.5$ and $1.0$ on
CIFAR-10, in
Table~\ref{tab:forgetting}, and plot the same values in
Figure~\ref{fig:taskwise_accuracy}. T0 is the oldest task, so forgetting appears
as low accuracy on the left of each row.

\noindent\textbf{Per-task accuracy.}\quad
The compared methods fall into three groups. FedAvg and FedEWC keep nothing of the
earlier tasks, and neither do the surrogate-aggregation FL methods FedDM and FedAF. Their accuracy is
at most $0.1\%$ on every task but the last, so their AA reduces to the accuracy
of the last task divided by the number of tasks. The other methods retain the
earlier tasks in part, and for most of them the accuracy rises steeply toward
the most recent tasks. FedCBDR on CIFAR-100 at $\beta{=}0.5$ reaches $66.9\%$ on T9 against
$22.1\%$ on T0, and Re-Fed on TinyImageNet at $\beta{=}0.5$ reaches $45.8\%$ on
T8 against $5.6\%$ on T0. On CIFAR-100 and TinyImageNet, GLFC and, at
$\beta{=}0.5$ and $1.0$, TARGET skew the other way and lose the newest tasks. At $\beta{=}1.0$, TARGET holds T0 at
$46.5\%$ on CIFAR-100 and $42.8\%$ on TinyImageNet, while its T9 falls to
$21.3\%$ and $6.4\%$. ReSCENE instead keeps every task within a high band. Its
lowest per-task accuracy exceeds the lowest of every baseline in all nine
settings, $36.2$ to $46.9\%$ against at most $24.4\%$ on CIFAR-10, $32.3$ to
$44.8\%$ against at most $17.8\%$ on CIFAR-100, and $19.9$ to $21.6\%$ against
at most $14.8\%$ on TinyImageNet. On CIFAR-10, it also holds the oldest task at
$75$ to $78\%$, where no baseline exceeds $39\%$.

\begin{table}[t]
\centering

\small
\setlength{\tabcolsep}{3pt}
\renewcommand{\arraystretch}{1.0}
{% (a) natural size, 2026-11-13
\begin{tabular}{@{}l|ccccc|ccccc|ccccc@{}}
\toprule
& \multicolumn{5}{c|}{$\beta{=}0.1$} & \multicolumn{5}{c|}{$\beta{=}0.5$} & \multicolumn{5}{c}{$\beta{=}1.0$}\\
Method & T0 & T1 & T2 & T3 & T4 & T0 & T1 & T2 & T3 & T4 & T0 & T1 & T2 & T3 & T4\\
\midrule
FedAvg  & 0.0 & 0.0 & 0.0 & 0.0 & 50.0 & 0.0 & 0.0 & 0.0 & 0.0 & 79.2 & 0.0 & 0.0 & 0.0 & 0.0 & 73.5\\
FedEWC  & 0.0 & 0.0 & 0.0 & 0.0 & 50.0 & 0.0 & 0.0 & 0.0 & 0.0 & 87.6 & 0.0 & 0.0 & 0.0 & 0.0 & 75.6\\
FedLwF  & 31.7 & 0.2 & 0.0 & 0.4 & 19.6 & 25.1 & 11.3 & 30.0 & 2.8 & 15.5 & 32.6 & 18.4 & 13.3 & 45.5 & 4.3\\
TARGET  & 35.8 & 0.0 & 0.2 & 0.4 & 13.4 & 38.7 & 6.7 & 20.9 & 22.2 & 19.6 & 29.0 & 18.9 & 6.3 & 30.6 & 9.5\\
MFCL    & 8.0 & 0.9 & 11.7 & 9.2 & 36.8 & 11.4 & 20.0 & 26.0 & 56.0 & 66.2 & 4.0 & 21.2 & 35.1 & 74.8 & 67.6\\
GLFC    & 0.4 & 0.2 & 3.3 & 20.5 & 50.1 & 4.9 & 4.3 & 16.8 & 58.2 & 92.7 & 13.9 & 7.0 & 28.9 & 74.3 & 91.6\\
FedCBDR & 1.7 & 1.3 & 3.8 & 59.9 & 63.0 & 21.2 & 11.7 & 34.6 & 73.4 & 93.0 & 23.9 & 12.6 & 43.2 & 77.4 & 93.2\\
Re-Fed  & 37.5 & 9.1 & 31.4 & 54.1 & 21.3 & 23.4 & 17.9 & 50.9 & 80.5 & 80.9 & 31.1 & 24.4 & 57.8 & 89.2 & 80.4\\
FedDM   & 0.0 & 0.0 & 0.0 & 0.0 & 78.5 & 0.0 & 0.0 & 0.0 & 0.0 & 88.6 & 0.0 & 0.0 & 0.0 & 0.0 & 90.5\\
FedAF   & 0.0 & 0.0 & 0.0 & 0.0 & 80.0 & 0.0 & 0.0 & 0.0 & 0.0 & 88.5 & 0.0 & 0.0 & 0.0 & 0.0 & 90.3\\
\rowcolor{oursbg}ReSCENE & 75.3 & 36.2 & 52.1 & 70.5 & 74.8 & 77.4 & 43.9 & 57.8 & 74.4 & 86.3 & 78.1 & 46.9 & 61.9 & 75.4 & 85.9\\
\bottomrule
\end{tabular}}

\smallskip
(a) CIFAR-10

\medskip
\resizebox{\textwidth}{!}{%
\begin{tabular}{@{}ll|cccccccccc|cccccccccc@{}}
\toprule
& & \multicolumn{10}{c|}{CIFAR-100} & \multicolumn{10}{c}{TinyImageNet}\\
Method & $\beta$ & T0 & T1 & T2 & T3 & T4 & T5 & T6 & T7 & T8 & T9 & T0 & T1 & T2 & T3 & T4 & T5 & T6 & T7 & T8 & T9\\
\midrule
FedAvg  & 0.1 & 0.0 & 0.0 & 0.0 & 0.0 & 0.0 & 0.0 & 0.0 & 0.0 & 0.0 & 58.0 & 0.0 & 0.0 & 0.0 & 0.0 & 0.0 & 0.0 & 0.0 & 0.0 & 0.0 & 48.1\\
FedEWC  & 0.1 & 0.0 & 0.0 & 0.0 & 0.0 & 0.0 & 0.0 & 0.0 & 0.0 & 0.0 & 58.4 & 0.0 & 0.0 & 0.0 & 0.0 & 0.0 & 0.0 & 0.0 & 0.0 & 0.0 & 48.0\\
FedLwF  & 0.1 & 3.1 & 0.2 & 0.8 & 5.8 & 7.0 & 9.0 & 9.5 & 24.3 & 8.3 & 22.1 & 2.3 & 6.6 & 4.5 & 8.2 & 10.1 & 6.2 & 11.3 & 15.1 & 21.9 & 23.8\\
TARGET  & 0.1 & 2.0 & 1.2 & 1.9 & 7.8 & 11.0 & 15.4 & 9.7 & 17.3 & 7.5 & 11.0 & 17.2 & 13.2 & 15.0 & 16.2 & 15.2 & 13.2 & 13.0 & 13.2 & 10.4 & 9.8\\
MFCL    & 0.1 & 4.4 & 4.0 & 5.3 & 7.7 & 9.4 & 13.5 & 19.0 & 15.4 & 18.1 & 37.6 & 7.6 & 11.3 & 10.0 & 11.9 & 9.2 & 10.6 & 10.5 & 12.8 & 17.4 & 20.3\\
GLFC    & 0.1 & 11.7 & 10.5 & 21.3 & 17.0 & 27.4 & 28.7 & 27.3 & 22.9 & 10.2 & 1.4 & 7.7 & 7.9 & 12.5 & 12.0 & 15.5 & 13.9 & 6.4 & 4.8 & 0.8 & 1.8\\
FedCBDR & 0.1 & 17.4 & 12.1 & 27.8 & 16.7 & 27.5 & 39.8 & 42.4 & 43.5 & 58.7 & 56.2 & 16.4 & 10.6 & 18.5 & 17.5 & 14.2 & 19.2 & 23.2 & 34.2 & 45.3 & 11.9\\
Re-Fed  & 0.1 & 12.1 & 13.1 & 24.6 & 20.9 & 26.1 & 39.2 & 44.5 & 43.9 & 47.8 & 23.6 & 1.1 & 1.7 & 3.1 & 7.3 & 9.4 & 14.8 & 22.8 & 33.0 & 36.9 & 19.8\\
FedDM   & 0.1 & 0.0 & 0.0 & 0.0 & 0.0 & 0.0 & 0.0 & 0.0 & 0.0 & 0.0 & 62.7 & 0.0 & 0.0 & 0.0 & 0.0 & 0.0 & 0.0 & 0.0 & 0.0 & 0.0 & 25.2\\
FedAF   & 0.1 & 0.0 & 0.0 & 0.0 & 0.0 & 0.0 & 0.0 & 0.0 & 0.0 & 0.0 & 63.7 & 0.0 & 0.0 & 0.0 & 0.0 & 0.0 & 0.0 & 0.0 & 0.0 & 0.0 & 25.8\\
\rowcolor{oursbg}ReSCENE & 0.1 & 48.6 & 38.5 & 51.7 & 42.6 & 46.9 & 51.4 & 52.1 & 42.5 & 42.3 & 32.3 & 27.6 & 24.3 & 26.2 & 26.1 & 19.9 & 23.1 & 21.4 & 25.1 & 23.8 & 20.9\\
\midrule
FedAvg  & 0.5 & 0.0 & 0.0 & 0.0 & 0.0 & 0.0 & 0.0 & 0.0 & 0.0 & 0.0 & 81.4 & 0.0 & 0.0 & 0.0 & 0.0 & 0.0 & 0.0 & 0.0 & 0.0 & 0.1 & 60.3\\
FedEWC  & 0.5 & 0.0 & 0.0 & 0.0 & 0.0 & 0.0 & 0.0 & 0.0 & 0.0 & 0.0 & 81.4 & 0.0 & 0.0 & 0.0 & 0.0 & 0.0 & 0.0 & 0.0 & 0.0 & 0.0 & 60.8\\
FedLwF  & 0.5 & 3.8 & 1.6 & 15.9 & 16.9 & 21.4 & 27.1 & 29.3 & 27.3 & 28.1 & 35.1 & 14.4 & 11.9 & 9.8 & 18.9 & 19.0 & 22.5 & 19.7 & 25.7 & 27.2 & 22.8\\
TARGET  & 0.5 & 29.9 & 6.8 & 16.5 & 19.0 & 20.8 & 27.1 & 23.9 & 22.9 & 19.6 & 15.3 & 41.3 & 16.2 & 23.0 & 29.4 & 19.8 & 26.6 & 18.4 & 19.3 & 12.6 & 8.4\\
MFCL    & 0.5 & 8.7 & 12.0 & 21.0 & 13.8 & 16.3 & 22.4 & 26.1 & 29.7 & 33.1 & 63.8 & 12.6 & 12.5 & 13.2 & 15.2 & 12.5 & 17.1 & 15.6 & 17.8 & 27.9 & 43.3\\
GLFC    & 0.5 & 18.1 & 16.9 & 28.3 & 20.8 & 26.6 & 27.8 & 27.5 & 18.3 & 15.4 & 2.5 & 14.8 & 12.6 & 15.1 & 13.8 & 16.6 & 14.0 & 7.8 & 8.3 & 1.4 & 2.0\\
FedCBDR & 0.5 & 22.1 & 17.1 & 35.4 & 25.2 & 34.8 & 44.3 & 49.1 & 54.3 & 67.2 & 66.9 & 19.8 & 14.8 & 21.1 & 22.3 & 18.7 & 24.0 & 25.7 & 35.7 & 46.4 & 29.6\\
Re-Fed  & 0.5 & 10.8 & 7.7 & 19.5 & 16.6 & 27.9 & 37.4 & 47.2 & 56.8 & 65.0 & 50.3 & 5.6 & 2.8 & 8.6 & 11.7 & 15.8 & 22.9 & 28.2 & 40.0 & 45.8 & 36.5\\
\rowcolor{oursbg}ReSCENE & 0.5 & 50.8 & 44.1 & 56.4 & 45.6 & 51.4 & 51.6 & 52.5 & 46.2 & 51.1 & 48.3 & 25.0 & 23.5 & 26.1 & 33.5 & 39.8 & 26.6 & 20.6 & 23.3 & 22.4 & 20.1\\
\midrule
FedAvg  & 1.0 & 0.0 & 0.0 & 0.0 & 0.0 & 0.0 & 0.0 & 0.0 & 0.0 & 0.0 & 82.4 & 0.0 & 0.0 & 0.0 & 0.0 & 0.0 & 0.0 & 0.0 & 0.0 & 0.1 & 62.1\\
FedEWC  & 1.0 & 0.0 & 0.0 & 0.0 & 0.0 & 0.0 & 0.0 & 0.0 & 0.0 & 0.0 & 82.7 & 0.0 & 0.0 & 0.0 & 0.0 & 0.0 & 0.0 & 0.0 & 0.0 & 0.0 & 62.3\\
FedLwF  & 1.0 & 5.7 & 10.2 & 18.0 & 19.6 & 19.4 & 33.3 & 33.1 & 28.8 & 34.6 & 33.3 & 9.1 & 6.6 & 9.4 & 18.0 & 25.6 & 24.7 & 23.6 & 22.4 & 28.1 & 24.8\\
TARGET  & 1.0 & 46.5 & 10.5 & 20.6 & 25.6 & 31.8 & 28.4 & 24.3 & 23.4 & 16.6 & 21.3 & 42.8 & 18.0 & 20.1 & 28.9 & 19.4 & 25.0 & 16.6 & 17.6 & 11.8 & 6.4\\
MFCL    & 1.0 & 10.3 & 12.1 & 23.2 & 16.1 & 21.6 & 24.7 & 27.4 & 28.3 & 37.0 & 67.1 & 11.5 & 11.2 & 12.0 & 16.1 & 11.7 & 14.7 & 14.5 & 19.2 & 32.5 & 49.8\\
GLFC    & 1.0 & 18.9 & 15.8 & 28.2 & 21.2 & 27.5 & 28.7 & 28.6 & 24.2 & 13.5 & 2.0 & 14.1 & 13.5 & 14.5 & 13.8 & 15.9 & 16.8 & 5.9 & 7.7 & 2.5 & 3.4\\
FedCBDR & 1.0 & 22.6 & 17.8 & 35.7 & 24.1 & 33.5 & 42.0 & 50.0 & 53.5 & 70.3 & 68.5 & 19.9 & 14.4 & 20.1 & 21.2 & 18.1 & 23.8 & 24.8 & 35.2 & 49.2 & 34.1\\
Re-Fed  & 1.0 & 8.9 & 8.8 & 20.4 & 16.6 & 31.7 & 37.3 & 49.0 & 55.8 & 64.4 & 53.5 & 5.1 & 4.5 & 7.2 & 12.4 & 17.1 & 26.5 & 29.1 & 40.5 & 48.9 & 37.8\\
\rowcolor{oursbg}ReSCENE & 1.0 & 50.8 & 44.8 & 56.2 & 46.6 & 51.2 & 53.5 & 54.1 & 47.9 & 52.9 & 49.0 & 26.5 & 24.4 & 26.0 & 26.5 & 21.6 & 26.0 & 22.5 & 36.1 & 44.8 & 32.5\\
\bottomrule
\end{tabular}}

\smallskip
(b) CIFAR-100 and TinyImageNet
\caption{Final accuracy (\%) on each task after the complete sequence. T0 is
the oldest task, and each row averages to the AA of the run.}
\label{tab:forgetting}
\end{table}

\noindent\textbf{Accuracy across the task sequence.}\quad
Figure~\ref{fig:taskwise_accuracy} plots the per-task accuracies of
Table~\ref{tab:forgetting} as curves over the task index, which shows how each
method distributes its accuracy between the old and the recent tasks. The curves
of FedAvg and FedEWC stay at zero and rise only at the final task, most of the
other curves climb toward the most recent tasks, and those of GLFC and, at
$\beta{=}0.5$ and $1.0$, TARGET decline toward them on CIFAR-100 and TinyImageNet. The curve of ReSCENE is
not flat, since it dips on harder tasks such as T1 on CIFAR-10 and varies by
$23$ points on TinyImageNet at $\beta{=}1.0$. It nevertheless stays high across
the whole sequence, and its lowest point lies above the lowest point of every
baseline in every panel, so no task is sacrificed regardless of when it arrived.
This follows from the server replaying every completed task in every round.

\begin{figure*}[t]
\centering
\includegraphics[width=0.94\textwidth]{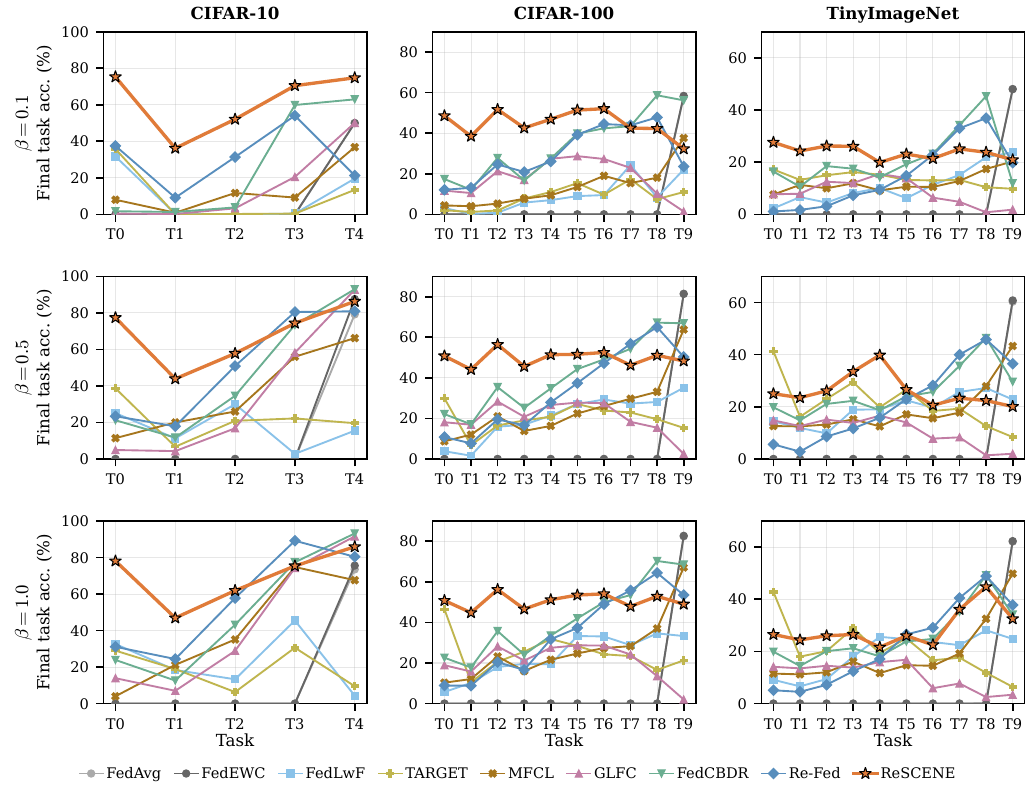}
\caption{Final accuracy of each task after the complete sequence for the methods
of Table~\ref{tab:main}, one column per benchmark and one row per $\beta$
(values in Table~\ref{tab:forgetting}).}
\label{fig:taskwise_accuracy}
\end{figure*}

\subsection{Domain-Incremental Learning}
\label{sec:domain_plan}

Federated incremental learning is defined in prior work by two scenarios, new
classes and new domains~\citep{Li_2024_CVPR}, and the main text covers the first.
We evaluate whether server-side surrogate replay also holds when each task
changes the visual domain while the label space stays fixed.

\begin{wraptable}{r}{0.50\textwidth}
\centering
\caption{Domain-incremental results on Office-31 (\%). Final accuracy on each
domain, AA, and AIA.}
\label{tab:office31}
\small
\renewcommand{\arraystretch}{1.05}
\setlength{\tabcolsep}{3pt}
\begin{tabular}{@{}lccc|cc@{}}
\toprule
& \multicolumn{3}{c|}{Domain} & \multicolumn{2}{c}{Accuracy}\\
Method & Amazon & DSLR & Webcam & AA & AIA\\
\midrule
FedAvg  & 18.83 & 23.23 & 65.19 & 35.75 & 36.21\\
FedEWC  & 19.18 & 21.21 & 61.39 & 33.93 & 35.94\\
FedLwF  & 25.75 & 37.37 & 56.96 & 40.03 & 37.07\\
TARGET  & 29.48 & 20.20 & 39.87 & 29.85 & 34.37\\
MFCL    & 22.56 & 42.42 & 63.29 & 42.76 & 41.66\\
GLFC    & 36.94 & 17.17 & 11.39 & 21.84 & 22.81\\
FedCBDR & 55.42 & 48.48 & 53.80 & 52.57 & 49.33\\
Re-Fed  & 48.31 & 42.42 & 53.80 & 48.18 & 47.77\\
\rowcolor{oursbg}ReSCENE & \textbf{56.48} & \textbf{65.66} & \textbf{76.58} & \textbf{66.24} & \textbf{56.10}\\
\bottomrule
\end{tabular}
\end{wraptable}

\noindent\textbf{Setting.}\quad
We learn the Amazon, DSLR, and Webcam domains of Office-31~\citep{Saenko2010AdaptingVC}
in sequence under a single shared $31$-class classifier. Following
Re-Fed~\citep{Li_2024_CVPR} we set $N{=}10$ clients, $K{=}4$ participants,
$\beta{=}1.0$, $20$ rounds per domain, and $20$ local epochs for the
model-training baselines, and everything else follows Table~\ref{tab:main}.
Since every domain shares the same classes, MFCL applies its cross-entropy on the
real samples to all $31$ classes instead of only the new ones.

\noindent\textbf{Results.}\quad
ReSCENE attains the highest accuracy on every domain (Table~\ref{tab:office31}), reaching $56.48\%$,
$65.66\%$, and $76.58\%$ on Amazon, DSLR, and Webcam, while FedCBDR reaches $55.42/48.48/53.80$ and
Re-Fed reaches $48.31/42.42/53.80$. This per-domain advantage suggests that, on Office-31,
uploading condensed surrogates and training a single model on the server is more
effective than training separate client models and aggregating their weights.
Furthermore, ReSCENE reaches $66.24\%$ AA and $56.10\%$ AIA, the best of all
methods. As it gathers the surrogates of all domains at the server and
trains one model over them, ReSCENE maintains high accuracy on the earlier
domains, preserving knowledge across domain transitions and limiting forgetting
as the input distribution shifts. ReSCENE thus remains effective under domain-incremental distribution shift.

\subsection{Comparison under a Matched Replay Budget}
\label{app:cost_memory}

As in ReSCENE, the replay baselines revisit data of earlier tasks during training,
so the amount of replay data could explain the gap between them. At the default
budgets of Table~\ref{tab:main}, ReSCENE already holds fewer replay images than
FedCBDR and TARGET and still outperforms them. For a fair comparison, we rerun
the replay baselines that store replay data with the same total number of replay
images as ReSCENE. MFCL samples its replay data from a generator at every step and
stores none, so it has no replay budget to match.

\noindent\textbf{Setting.}\quad
We use CIFAR-10 at $\beta{=}0.5$ with $N{=}20$ clients and count the replay images
that the methods train on at the end of the sequence, held by all clients for the
baselines and by the server for ReSCENE. Each method has one
memory parameter, which Table~\ref{tab:memory_matched} lists as Budget. It is the
number of exemplars per class on every client for GLFC and FedCBDR, the number of
synthetic images that each client receives per task for TARGET, the total storage of a client for Re-Fed,
and the per-class server buffer $B$ for ReSCENE. We set each budget so that the
total becomes ReSCENE's $10{,}000$ images. TARGET keeps its synthesis schedule
while each client receives a $500$-image subset, and Re-Fed already holds
$10{,}253$ images at its default and is left unchanged.

\begin{table}[h]
\centering

\small
\setlength{\tabcolsep}{5pt}
\renewcommand{\arraystretch}{1.08}
\begin{tabular}{@{}l|rrrrr|rrrr@{}}
\toprule
& \multicolumn{5}{c|}{Setting of Table~\ref{tab:main}}
& \multicolumn{4}{c}{Matched setting (10,000 in total)}\\
\cmidrule(lr){2-6}\cmidrule(lr){7-10}
Method & Budget & Per client & Server $+$ & AA & AIA & Budget & Server $+$ & AA & AIA\\
       &        &            & clients    &    &     &        & clients    &    &    \\
\midrule
TARGET   & 25,600 & 25,600    & 512,000 & 21.66 & 36.76 & 500     & 10,000 & 19.62 & 37.62\\
GLFC     & 20       & 200         & 4,000   & 35.38 & 48.90 & 50      & 10,000 & 40.67 & 51.69\\
FedCBDR  & 150      & 1,500     & 30,000  & 46.78 & 57.64 & 50      & 10,000 & 32.85 & 49.86\\
Re-Fed   & 1,000  & $\approx$513 & 10,253  & 50.73 & 62.69 & 1,000 & 10,253 & 50.73 & 62.69\\
\rowcolor{oursbg}ReSCENE  & 1,000  & 0           & 10,000  & \textbf{67.96} & \textbf{73.53} & 1,000 & 10,000 & \textbf{67.96} & \textbf{73.53}\\
\bottomrule
\end{tabular}
\caption{Replay images and accuracy on CIFAR-10 ($\beta{=}0.5$) at the budgets
of Table~\ref{tab:main} and after matching every baseline to ReSCENE's
$10{,}000$ images. Budget denotes the memory parameter of each method, defined in the
text.}
\label{tab:memory_matched}
\end{table}

\noindent\textbf{Results.}\quad
With the same $10{,}000$ replay images, ReSCENE outperforms these baselines
by $17.2$ to $48.3$ AA points and by $10.8$ to $35.9$ AIA points
(Table~\ref{tab:memory_matched}).
GLFC gains $5.29$ AA points when its memory grows to $50$ exemplars per class, and
FedCBDR loses $13.93$ points when its memory shrinks to the same size, which shows
how strongly client-side replay depends on the amount of stored data. In the
exemplar-based methods, the budget is split across $20$ clients, so each client
replays only a small, skewed cache of its own data, and the knowledge of these
caches is merged only through weight averaging. ReSCENE instead keeps a single copy
of the replay data at the server and trains on it together with the current task
in one objective at every round. This points to where and how the replay data are
used as the source of its advantage, which relies on a resource-rich server that
stores the buffer and trains the model.

\subsection{Additional Herding-Buffer Sweeps}
\label{sec:lambda_sweep}
\label{sec:policy_ablation}

Temporal herding has two hyperparameters, the temporal fraction $\lambda$ of the
recent rounds that defines the target and the per-class budget $B$. The main text fixes
them at $\lambda{=}0.75$ and $B{=}1000$, and we examine how each of them affects
the accuracy.

\noindent\textbf{Setting.}\quad
On CIFAR-10 at $\beta{=}0.5$, we vary $\lambda\in\{0.25,0.5,0.75,1\}$ at
$B{=}1000$ and $B\in\{500,1000,1500,2000\}$ at $\lambda{=}0.75$, keeping the rest of
the protocol. The whole task pool stays eligible for selection in every run, and
only the target or the budget changes, so $\lambda{=}1$ becomes full-pool herding.

\noindent\textbf{Temporal fraction.}\quad
$\lambda{=}0.75$ attains the highest AA and AIA among the four values
(Table~\ref{tab:buffer_budget}a). A small $\lambda$ defines the target from only
the last few rounds, so the target rests on few surrogates and can be biased toward
the participants of those rounds, which lowers AA to $66.92$ at $\lambda{=}0.25$.
At the other end, $\lambda{=}1$ averages over every round, including the stale early
surrogates, and reduces to full-pool herding with $66.70$ AA. The intermediate
window removes the earliest rounds from the target while still averaging over
enough surrogates. This is consistent with the analysis of Sec.~\ref{sec:theory_short}, where
moving the target to the recent window lowers the stale bias from $A_P$ to $A_W$
(Corollary~\ref{cor:fullpool} in Appendix~\ref{app:theory_full}), while a window that is too narrow leaves the target to
the scatter of a few rounds.

\begin{wraptable}{r}{0.40\textwidth}
\centering
\caption{Temporal herding sweeps on CIFAR-10 ($\beta{=}0.5$). (a) varies
$\lambda$ at $B{=}1000$, and (b) varies $B$ at $\lambda{=}0.75$. Bold marks the
default.}
\label{tab:buffer_budget}
\small
\setlength{\tabcolsep}{4pt}
\renewcommand{\arraystretch}{1.08}
{% natural size, 2026-11-13
\begin{tabular}{@{}ccc@{\hspace{10pt}}ccc@{}}
\toprule
\multicolumn{3}{c}{(a) Temporal fraction} & \multicolumn{3}{c}{(b) Budget}\\
\cmidrule(r{6pt}){1-3}\cmidrule(l){4-6}
$\lambda$ & AA & AIA & $B$ & AA & AIA\\
\midrule
0.25 & 66.92 & 73.15 & 500  & 61.11 & 69.12\\
0.50 & 67.16 & 73.35 & \textbf{1000} & 67.96 & 73.53\\
\textbf{0.75} & \textbf{67.96} & \textbf{73.53} & 1500 & 69.43 & 74.07\\
1.00 & 66.70 & 72.86 & 2000 & 69.19 & 74.67\\
\bottomrule
\end{tabular}}
\end{wraptable}

\noindent\textbf{Per-class budget.}\quad
Raising $B$ from $500$ to $1000$ improves AA by $6.85$ points
(Table~\ref{tab:buffer_budget}b), in line with Lemma~\ref{thm:greedy_target_bound},
whose bound $\Delta_W/\sqrt B$ on the error of summarizing the target is $41\%$
larger at $B{=}500$. On CIFAR-10, a class collects at most $2{,}000$ surrogates in a task and about
$1{,}700$ on average, so $B{=}2000$ keeps the whole pool and coincides with full
accumulation. The bounded buffer at $B{=}1000$ stays within $1.5$ AA points of the
larger budgets while holding $40\%$ fewer surrogates than the whole pool, so
temporal herding summarizes each task with less server memory and replay
computation than the whole pool. On CIFAR-100 and TinyImageNet, the bounded buffer
is also more accurate than full accumulation (Sec.~\ref{sec:herding_main}). We therefore adopt $B{=}1000$ as the
default, which balances accuracy against the recurring server cost.

\subsection{Ablation of Imbalance Resolution}
\label{app:balance_logit_ablation}

The server's training set $\B_{0:t-1}\cup S_t$ holds different numbers of
surrogates per class, since every earlier class is capped at $B$ surrogates while
the current task receives new surrogates in every round, so the current classes
outnumber the earlier ones in the later rounds of a task. ReSCENE counters this skew with two
corrections. First, it samples classes with probability
$\omega_c=n_c^\alpha/\sum_{c'}n_{c'}^\alpha$, where $n_c$ is the class count in this
set and $\alpha\in[0,1]$ spans count-proportional sampling at $\alpha=1$ and
class-balanced sampling at $\alpha=0$. Second, it trains with logit
adjustment~\citep{menon2021longtail} under the class prior of this set
$\pi_c=n_c/\sum_{c'}n_{c'}$, minimizing
$\mathcal L_{\mathrm{server}}=\mathrm{CE}\!\left(f_\theta(x)+\tau\log\boldsymbol\pi,\,y\right)$
with $c\sim\omega$. We measure how much each of them contributes.

\begin{wraptable}{r}{0.42\textwidth}
\centering
\caption{Ablation of class-imbalance resolution on CIFAR-10 ($\beta{=}0.5$).}
\label{tab:balance_logit_ablation}
\small
\renewcommand{\arraystretch}{1.08}
\setlength{\tabcolsep}{4pt}
\begin{tabular}{@{}lcc|cc@{}}
\toprule
& & & \multicolumn{2}{c}{Accuracy}\\
Variant & $\alpha$ & $\tau$ & AA & AIA\\
\midrule
Natural                & 1.0 & 0.0 & 67.26 & 72.55\\
Sampling only          & 0.5 & 0.0 & 67.64 & 72.65\\
Logit only             & 1.0 & 1.0 & 67.70 & 73.26\\
\textbf{ReSCENE}       & 0.5 & 1.0 & \textbf{67.96} & \textbf{73.53}\\
\bottomrule
\end{tabular}
\end{wraptable}

\noindent\textbf{Setting.}\quad
On CIFAR-10 at $\beta{=}0.5$, we switch each correction on and off through its
coefficient. The sampling exponent $\alpha$ sets how often each class enters a
batch, from count-proportional sampling at $\alpha{=}1$, which follows the skewed
pool, to square-root sampling at $\alpha{=}0.5$. The logit-adjustment strength
$\tau$ sets how strongly the classifier is corrected by the pool prior, from no
correction at $\tau{=}0$ to $\tau{=}1$. Since the batches already follow $\omega$,
adjusting by $\boldsymbol\pi$ shifts the classifier further toward the less frequent
classes than the sampled prior alone requires. Natural training
($\alpha{=}1$, $\tau{=}0$) uses neither correction, and ReSCENE
($\alpha{=}0.5$, $\tau{=}1$) uses both.

\noindent\textbf{Results.}\quad
Natural training already reaches $67.26$ AA and $72.55$ AIA
(Table~\ref{tab:balance_logit_ablation}), so most of ReSCENE's accuracy comes from
the server-side replay itself. Each correction alone improves on it, square-root
sampling by $0.38$ AA and $0.10$ AIA points and logit adjustment by $0.44$ AA and
$0.71$ AIA points. The two act on different parts of training, sampling on the
data that enters each batch and logit adjustment on the bias of the classifier,
and applying both gives the highest AA and AIA, $0.70$ and $0.98$ points above
natural training.

\section{System Cost Analysis}
\label{app:cost}

In weight-communicating methods, each client trains the model on its local data
and uploads it, and the server only aggregates the updates. In ReSCENE, the
clients only condense surrogates, and most of the computation takes place at the
server. We therefore analyze its computation and communication cost in more depth.
Sec.~\ref{sec:resource} analyzes the cost on the client side, and
Sec.~\ref{sec:total_compute} analyzes the total cost of the whole federated system,
including the server.

\subsection{Client-Side Resource Accounting}
\label{sec:resource}

ReSCENE optimizes and communicates surrogates in place of training and
communicating the model, so its clients should spend less computation and less
bandwidth than under any weight-communicating method. We measure
every client-side resource, namely computation, persistent state, uplink, and
downlink, for all nine methods on the three benchmarks.

\begin{table}[h!]
\centering

\small
\setlength{\tabcolsep}{6pt}
\renewcommand{\arraystretch}{1.05}
\begin{tabular}{@{}ll cccc@{}}
\toprule
Dataset & Method & Client comp. & State mem. & Uplink & Downlink\\
        &        & (PFLOP/run)  & (MB)       & (MB/rd) & (MB/rd)\\
\midrule
\multirow{9}{*}{CIFAR-10}
 & FedAvg  & 32.68 {\scriptsize($\times$1.0)}  & 44.73   & 44.73 & 44.73\\
 & FedEWC  & 34.31 {\scriptsize($\times$1.0)}  & 134.13  & 44.73 & 44.73\\
 & FedLwF  & 41.32 {\scriptsize($\times$1.3)}  & 89.47   & 44.73 & 44.73\\
 & TARGET  & 72.20 {\scriptsize($\times$2.2)}  & 359.31  & 44.73 & 359.31\\
 & MFCL    & 111.34 {\scriptsize($\times$3.4)} & 96.97   & 44.73 & 52.22\\
 & GLFC    & 48.43 {\scriptsize($\times$1.5)}  & 91.93   & 44.73 & 44.73\\
 & FedCBDR & 72.71 {\scriptsize($\times$2.2)}  & 63.17   & 44.73 & 44.73\\
 & Re-Fed  & 71.70 {\scriptsize($\times$2.2)}  & 57.02   & 44.73 & 44.73\\
 & \ours ReSCENE & \ours 3.59 {\scriptsize($\times$0.1)} & \ours 44.73 & \ours 0.25 & \ours 44.73\\
\midrule
\multirow{9}{*}{CIFAR-100}
 & FedAvg  & 33.24 {\scriptsize($\times$1.0)}   & 44.92   & 44.92 & 44.92\\
 & FedEWC  & 34.90 {\scriptsize($\times$1.0)}   & 134.68  & 44.92 & 44.92\\
 & FedLwF  & 43.20 {\scriptsize($\times$1.3)}   & 89.84   & 44.92 & 44.92\\
 & TARGET  & 84.30 {\scriptsize($\times$2.5)}   & 359.49  & 44.92 & 359.49\\
 & MFCL    & 132.45 {\scriptsize($\times$4.0)}  & 97.53   & 44.92 & 52.40\\
 & GLFC    & 203.18 {\scriptsize($\times$6.1)}  & 114.41  & 44.92 & 44.92\\
 & FedCBDR & 933.15 {\scriptsize($\times$28.1)} & 229.24  & 44.92 & 44.92\\
 & Re-Fed  & 146.46 {\scriptsize($\times$4.4)}  & 57.21   & 44.92 & 44.92\\
 & \ours ReSCENE & \ours 27.49 {\scriptsize($\times$0.8)} & \ours 44.92 & \ours 1.23 & \ours 44.92\\
\midrule
\multirow{9}{*}{TinyImageNet}
 & FedAvg  & 67.21 {\scriptsize($\times$1.0)}    & 45.12    & 45.12 & 45.12\\
 & FedEWC  & 70.58 {\scriptsize($\times$1.1)}    & 135.30   & 45.12 & 45.12\\
 & FedLwF  & 87.38 {\scriptsize($\times$1.3)}    & 90.25    & 45.12 & 45.12\\
 & TARGET  & 158.60 {\scriptsize($\times$2.4)}   & 1303.42  & 45.12 & 1303.42\\
 & MFCL    & 264.96 {\scriptsize($\times$3.9)}  & 98.73    & 45.12 & 53.20\\
 & GLFC    & 410.42 {\scriptsize($\times$6.1)}   & 286.86   & 45.12 & 45.12\\
 & FedCBDR & 1884.40 {\scriptsize($\times$28.0)} & 1519.68  & 45.12 & 45.12\\
 & Re-Fed  & 295.96 {\scriptsize($\times$4.4)}   & 143.43   & 45.12 & 45.12\\
 & \ours ReSCENE & \ours 56.20 {\scriptsize($\times$0.8)} & \ours 45.12 & \ours 9.83 & \ours 45.12\\
\bottomrule
\end{tabular}
\caption{Client-side resources with ResNet-18. Multipliers are relative to
FedAvg on the same dataset.}
\label{tab:resource}
\end{table}

\noindent\textbf{Setting.}\quad
Table~\ref{tab:resource} covers one complete run of the Table~\ref{tab:main}
protocol, $100$ rounds on CIFAR-10 and $200$ on CIFAR-100 and TinyImageNet with
$10$ participants per round, and counts compute as floating-point operations with
backward passes at twice the cost of forward passes. Client compute is the cost
of all client-side work over the run, summed over the participating clients.
State memory is the persistent client state between rounds, the model together
with any auxiliary model or replay data, and excludes transient autograd and
optimizer workspace. Uplink and downlink each measure one message per round.

\noindent\textbf{Analysis.}\quad
ReSCENE uses the least client computation and uplink on every benchmark and keeps
the smallest client state, tied with FedAvg (Table~\ref{tab:resource}). Its client
computation is $0.1\times$ that of FedAvg on CIFAR-10 and $0.8\times$ on CIFAR-100
and TinyImageNet, whereas every replay baseline exceeds FedAvg, since its clients
train on replay data in addition to the new data. FedCBDR, which keeps
$150$ samples per class, reaches $28.0\times$ the client cost of FedAvg on
TinyImageNet, and GLFC and Re-Fed reach $6.1\times$ and $4.4\times$. The clients of
ReSCENE store only the model, as in FedAvg, whereas every continual-learning
baseline also stores an auxiliary model or replay data, up to $1{,}519.68$ MB for
FedCBDR on TinyImageNet. Its uplink is $0.25$ MB on CIFAR-10 and $9.83$ MB on
TinyImageNet against the $44.73$ to $45.12$ MB model update of every other method,
while its downlink equals that of FedAvg, since the server broadcasts the global
model. These savings follow from the design of ReSCENE, whose clients only
condense a fixed number of surrogates per class through the frozen global model,
keep no replay data, and upload the surrogates in place of the model.

\subsection{Total Computation across the Server and Clients}
\label{sec:total_compute}

Cross-device federated learning pairs many resource-constrained clients with a
resource-rich server~\citep{10.1561/2200000083}. ReSCENE moves the training from
these clients to the server and thus shifts the computational burden from the
clients to the server. We therefore compare the total computation of
the whole system, summing the server and all clients.

\begin{table}[h!]
\centering

\footnotesize
\setlength{\tabcolsep}{4pt}
\renewcommand{\arraystretch}{1.3}
\resizebox{\textwidth}{!}{%
\begin{tabular}{l cc | ccccc}
\toprule
& \multicolumn{2}{c|}{Asymptotic compute (per task)}
& \multicolumn{5}{c}{CIFAR-10 (per task)}\\
\cmidrule(lr){2-3}\cmidrule(lr){4-8}
Method & All clients & Server & All clients & Server & Total & Clients & Server\\
       &             &        & (PFLOP)     & (PFLOP) & (PFLOP) & share & share\\
\midrule
FedAvg  & $O(R K E D_{\mathrm{loc}} F)$ & $O(R K |\theta|)$
        & 6.54 & $\approx$0 & 6.54 & 100\% & 0\%\\
FedEWC  & $O(R K E D_{\mathrm{loc}} F)$ & $O(R K |\theta|)$
        & 6.86 & $\approx$0 & 6.86 & 100\% & 0\%\\
FedLwF  & $O(R K E D_{\mathrm{loc}} F)$ & $O(R K |\theta|)$
        & 8.26 & $\approx$0 & 8.26 & 100\% & 0\%\\
TARGET  & $O(R K E D_{\mathrm{loc}} F)$
        & $O(R_{\mathrm{syn}} S_{\mathrm{kd}} B_{\mathrm{syn}} F)$
        & 14.44 & 38.2 & 52.64 & 27.4\% & 72.6\%\\
MFCL    & $O(R K E D_{\mathrm{loc}} (F{+}F_{\mathrm{gen}}))$
        & $O(I_{\mathrm{gen}} B_{\mathrm{gen}} (F{+}F_{\mathrm{gen}}))$
        & 22.27 & 2.58 & 24.85 & 89.6\% & 10.4\%\\
GLFC    & $O(R K E (D_{\mathrm{loc}}{+}D_{\mathrm{mem}}) F)$ & $O(R K |\theta| + I K |\Cset_{t-1}| F_{\mathrm{enc}})$
        & 9.69 & $<$0.01 & 9.69 & $>$99\% & $<$1\%\\
FedCBDR & $O(R K E (D_{\mathrm{loc}}{+}D_{\mathrm{mem}}) F + N D_{\mathrm{loc}} F)$ & $O(R K |\theta| + N D_{\mathrm{loc}} d^2)$
        & 14.54 & $\approx$0 & 14.54 & 100\% & 0\%\\
Re-Fed  & $O(R K E (D_{\mathrm{loc}}{+}D_{\mathrm{mem}}) F + N E_{\mathrm{PIM}} D_{\mathrm{pool}} F)$ & $O(R K |\theta|)$
        & 14.34 & $\approx$0 & 14.34 & 100\% & 0\%\\
\rowcolor{oursbg}ReSCENE & $O(R K J\,\mathrm{IPC}\,|\Cset_t| F)$ & $O(R E_{\mathrm{srv}} |\B{\cup}S| F)$
        & 0.72 & 23.7 & 24.41 & 2.9\% & 97.1\%\\
\bottomrule
\end{tabular}
}
\caption{Total computation across the server and all clients, asymptotically per
task and on CIFAR-10. The client columns average the full run over the five tasks,
and the server columns are measured at task~1.}
\label{tab:total_compute}
\end{table}

\noindent\textbf{Setting.}\quad
Table~\ref{tab:total_compute} sums computation over the server and all $K$
participating clients, asymptotically and on CIFAR-10 at task~1 with one completed
task's buffer, using $E_{\mathrm{srv}}{=}100$ server epochs per round,
$R_{\mathrm{syn}}{=}100$ synthesis rounds, $I_{\mathrm{gen}}{=}5{,}000$ generator
iterations, and the client figures of
Table~\ref{tab:resource} divided over the five tasks. In the asymptotic columns,
$R$ is the rounds per task, $N$ the total clients, $E$ and $E_{\mathrm{srv}}$ the
local and server epochs, $E_{\mathrm{PIM}}$ Re-Fed's personalized-informative-model
epochs, $D_{\mathrm{loc}}$ the local images of a client, $D_{\mathrm{mem}}$ the
replay images it stores (the exemplars of GLFC and FedCBDR and the cache of
Re-Fed), $D_{\mathrm{pool}}$ the previous-sample pool Re-Fed scores at a task
boundary, $F$ the per-image forward-and-backward cost, $F_{\mathrm{enc}}$ the same
for GLFC's LeNet-scale encoding model, $I$ its gradient-inversion steps per
prototype, $|\Cset_{t-1}|$ the classes of the previous task, $d$ the feature
dimension of the leverage scores FedCBDR computes by SVD on the server,
$S_{\mathrm{kd}}$ TARGET's distillation steps per synthesis round,
$B_{\mathrm{syn}}$ its synthesis batch, $I_{\mathrm{gen}}$ and $B_{\mathrm{gen}}$
MFCL's generator-training iterations and batch at each task boundary,
$F_{\mathrm{gen}}$ the per-image cost of its generator, $|\theta|$ the number of model parameters,
$J$ the condensation steps, and $|\B{\cup}S|$ the buffer-and-pool size.

\noindent\textbf{Analysis.}\quad
The client computation of the weight-communicating methods grows with the local
data $D_{\mathrm{loc}}$ of every participant, and with the replay data
$D_{\mathrm{mem}}$ for GLFC, FedCBDR, and Re-Fed, whereas that of ReSCENE depends
only on the condensation budget $\mathrm{IPC}{\times}|\Cset_t|$. ReSCENE therefore
gains the most when clients hold much local data or when a replay baseline stores
a large memory. Its server term grows with $E_{\mathrm{srv}}$ and the size
$|\B{\cup}S|$ of the buffer and pool, which grows with the participants only
through the $\mathrm{IPC}$ surrogates per class that each of them uploads. Adding
local data thus changes neither term of ReSCENE, and its server cost can be reduced
through $E_{\mathrm{srv}}$.

On CIFAR-10, the clients of ReSCENE spend $0.72$ PFLOP per task, $11\%$ of FedAvg's
$6.54$ PFLOP and $5\%$ of Re-Fed's $14.34$ PFLOP, and they carry only $2.9\%$ of the
total computation of ReSCENE. The remaining $97.1\%$ runs on the resource-rich
server. The total of $24.41$ PFLOP is about
$3.7\times$ that of FedAvg. TARGET also works at the server,
synthesizing data of earlier tasks by model inversion, but this costs
it $38.2$ PFLOP on the server, and its clients still spend $14.44$ PFLOP. At this task, ReSCENE
spends less on both sides, $23.7$ against $38.2$ PFLOP on the server and $0.72$
against $14.44$ PFLOP on the clients, and $24.41$ against $52.64$ PFLOP in total.
The server cost of ReSCENE grows with the buffer in later tasks, and over the full
run ReSCENE still spends less in total than TARGET, about $193$ against $224$
PFLOP.

\noindent\textbf{Round latency.}\quad
In this work, we assume synchronous federated learning, in which the clients wait for the server
to broadcast the next model. In weight-communicating methods, the server only
averages the updates, so a round lasts as long as the slowest client takes to train
locally and upload its model. In ReSCENE, the server trains the global model for
$E_{\mathrm{srv}}$ epochs before the next broadcast, which adds a server phase to
every round, whereas each client only condenses its surrogates and uploads up to
$179\times$ less data. ReSCENE thus moves the time of a round from the slowest
client to the resource-rich server.

\section{Privacy Analysis for ReSCENE}
\label{app:privacy}

In ReSCENE, the labelled surrogates are the only quantity that leaves a client.
They are the same quantity that FedDM~\citep{Xiong_2023_CVPR} communicates, so the
private condensation of FedDM applies to ReSCENE. When differential privacy is
required, a client initializes the surrogates from Gaussian noise and optimizes
them by DP-SGD~\citep{10.1145/2976749.2978318}. The noise is added at the
client, so the resulting guarantee holds against the server and any observer of
the uploads or of quantities computed from them. We first state this guarantee
and prove it in Sec.~\ref{app:privacy_formal}. Sec.~\ref{app:privacy_proc} then
gives the algorithm that a client runs and the privacy budget that it spends. In
Sec.~\ref{app:privacy_emp}, we compare the accuracy of ReSCENE with the baselines
under the same privacy budget, and Sec.~\ref{app:privacy_vis} visualizes the
surrogates that the server receives.

\subsection{Differential Privacy Guarantee of the Surrogate Upload}
\label{app:privacy_formal}

\noindent\textbf{Differential privacy.}\quad
A randomized mechanism $\mathcal{M}$ is $(\epsilon,\delta)$-differentially
private~\citep{10.1561/0400000042} if, for any two datasets $D,D'$ that differ in
a single record and any set of outcomes $O$,
\begin{equation}
\Pr[\mathcal{M}(D)\in O]\;\le\;e^{\epsilon}\,\Pr[\mathcal{M}(D')\in O]+\delta .
\label{eq:dp}
\end{equation}
We consider instance-level privacy, where $D$ is the local data of one client and
$D'$ is obtained from $D$ by adding or removing a single image. One image then
changes the distribution of what the client uploads by at most a factor
$e^{\epsilon}$, up to the slack $\delta$, and a smaller $\epsilon$ gives a stronger
protection.

\noindent\textbf{DP-SGD.}\quad
DP-SGD~\citep{10.1145/2976749.2978318} makes gradient descent differentially
private. At every step, it includes each example independently with probability
$q$, clips the gradient of each included example to norm $C$, and adds Gaussian
noise $\mathcal{N}(0,\sigma^2C^2\mathbf{I})$ to the sum of the clipped gradients,
where $\sigma$ is the noise multiplier. Each step is a subsampled Gaussian
mechanism, and the privacy of $J$ steps follows from composing them. The following
theorem gives the asymptotic form of this guarantee for small $q$.

\begin{theorem}[Privacy of DP-SGD~\citep{10.1145/2976749.2978318}]
\label{thm:dpsgd}
There exist constants $c_1$ and $c_2$ such that, given the sampling probability $q$
and the number of steps $J$, for any $\epsilon<c_1q^2J$, DP-SGD is
$(\epsilon,\delta)$-differentially private for any $\delta>0$ if
$\sigma\ge c_2\,q\sqrt{J\log(1/\delta)}/\epsilon$.
\end{theorem}

\noindent For given $q$, $\sigma$, and $J$, we denote by
$\epsilon(q,\sigma,J;\delta)$ the bound obtained by composing the R\'enyi
differential privacy~\citep{8049725} of $J$ subsampled Gaussian
mechanisms~\citep{mironov2019renyidifferentialprivacysampled} and converting it to
$(\epsilon,\delta)$-differential privacy with the tightened
conversion~\citep{pmlr-v108-balle20a}. This bound
has no closed form and is evaluated numerically, and the values that we report are
computed in this way.

\noindent\textbf{Condensation as DP-SGD.}\quad
Consider one condensation step for class $c$ under the network initialized at that
step, with embedding $\phi$. The feature-matching term of the client loss is
$\lVert\mu^{\mathrm{real}}_c-\mu^{\mathrm{syn}}_c\rVert^2$, where
$\mu^{\mathrm{real}}_c=\tfrac{1}{|\mathcal{Q}|}\sum_{i\in\mathcal{Q}}\phi(x_i)$ is
the mean feature of a set $\mathcal{Q}$ of real images and
$\mu^{\mathrm{syn}}_c$ is the mean feature of the surrogates $S$ of the class. As
FedDM shows, the gradient of this term with respect to the surrogate pixels
decomposes into one term per real image,
\begin{equation}
\nabla_{S}\lVert\mu^{\mathrm{real}}_c-\mu^{\mathrm{syn}}_c\rVert^2
=\frac{1}{|\mathcal{Q}|}\sum_{i\in\mathcal{Q}}
2\,\big(\partial_{S}\mu^{\mathrm{syn}}_c\big)^{\!\top}
\big(\mu^{\mathrm{syn}}_c-\phi(x_i)\big) ,
\label{eq:decomp}
\end{equation}
and the logit-matching term decomposes in the same way. We write $\tilde g(x_i)$
for the summand of image $x_i$, summed over the feature-matching and
logit-matching terms. The network normalizes every image separately, so
$\phi(x_i)$ and $\tilde g(x_i)$ depend on no real image other than $x_i$, and the
surrogates take the role that the model parameters take in DP-SGD.

\noindent\textbf{Private condensation.}\quad
A client privatizes every condensation step as DP-SGD does. For each class of the
current task, it draws a lot $\mathcal{Q}$ that includes each of its images of the class
independently with probability $q$, clips every per-image term to norm $C$, adds
Gaussian noise to the sum, and divides by a public normalization constant $L$,
\begin{equation}
\widehat{g}
=\frac{1}{L}\Big(\sum_{i\in\mathcal{Q}}
\frac{\tilde g(x_i)}{\max\!\big(1,\lVert\tilde g(x_i)\rVert_2/C\big)}
+\mathcal{N}(0,\sigma^2C^2\mathbf{I})\Big).
\label{eq:dp_step}
\end{equation}
Each image belongs to one class, so it enters the update of one surrogate set,
with a contribution of norm at most $C$. A client runs these steps for every class
of the task, with empty lots for a class of which it holds no image, so the number
of surrogates that it uploads is fixed by $\mathrm{IPC}$ and the public class set. The values $q$, $L$, $C$, $\sigma$, and
$J$ are public constants shared by all clients. Algorithm~\ref{alg:dp} in
Sec.~\ref{app:privacy_proc} lists the procedure.

\begin{proposition}[Differential privacy of private condensation in ReSCENE]
\label{prop:dp_rescene}
Suppose that every client condenses its local data by Algorithm~\ref{alg:dp}, which
it runs for every class of the current task, with surrogates initialized from
random noise, $J$ steps of Eq.~\eqref{eq:dp_step}, and a freshly initialized
network at every step that embeds each image independently of the other images.
Then, for every $\delta\in(0,1)$, the surrogates that a client uploads in a
communication round are
$(\epsilon,\delta)$-differentially private with respect to the images of that
client, where $\epsilon=\epsilon(q,\sigma,J;\delta)$. Every quantity that the
server computes from the uploads, including the buffers selected by temporal
herding and the global model, is a post-processing of the uploads received so far
and inherits their composed guarantee, which is given below for the rounds of a
task.
\end{proposition}

\noindent\textit{Proof.}
Fix a client and a round. The initialization of the surrogates from random noise,
the sampling of the lots, and the initialization of the network at each step use
randomness that the client draws independently of the local data. The clipping
bounds every per-image term by $C$ for any network. For each class, every step of Eq.~\eqref{eq:dp_step} is therefore a
subsampled Gaussian mechanism with sampling probability $q$ and noise multiplier
$\sigma$, and the momentum update of the surrogates is a post-processing of its
output. The client uploads $\mathrm{IPC}$ surrogates for every class of the task,
whether or not it holds an image of the class, so the uploads for $D$ and $D'$
have the same form. Adding or removing one image changes the lots of one class
only, so, by
adaptive composition, the $J$ steps of that class are
$(\epsilon,\delta)$-differentially private with
$\epsilon=\epsilon(q,\sigma,J;\delta)$, and the surrogates of the other classes are
unaffected. The clients hold disjoint data, so by parallel
composition~\citep{10.1145/1559845.1559850}, the same instance-level guarantee
holds for the uploads of all clients in the round, as in FedDM. The server
accumulates the uploads, compresses each task into a buffer by temporal herding,
and trains the global model on the buffer and the pool without accessing the data
of a client again. By post-processing~\citep[Prop.~2.1]{10.1561/0400000042}, these
computations consume no additional privacy budget. \hfill$\square$

\noindent\textbf{Rounds and tasks.}\quad
Proposition~\ref{prop:dp_rescene} states the guarantee of one round, which is the
unit that FedDM reports. A client uses the images of a task in each of the $R$
rounds of that task, so the steps of these rounds compose adaptively, and the
uploads of a task are $(\epsilon(q,\sigma,RJ;\delta),\delta)$-differentially
private. Each image belongs to one task and a client
uses it in the rounds of that task alone, so by parallel composition over tasks,
the budget of one task does not accumulate into the next.

\noindent\textbf{Server-side reuse consumes no further budget.}\quad
The forgetting mitigation of ReSCENE takes place entirely at the server after
communication. The server stores the uploaded surrogates, compresses them into
buffers, and replays these buffers in every later round, and none of these steps
accesses the data of a client again. The global model that the server broadcasts is
also derived from the surrogates alone. By Proposition~\ref{prop:dp_rescene}, the
storage and replay consume no further privacy budget, and the budget is spent only
at upload time.

\subsection{Private Condensation Algorithm and Its Privacy Budget}
\label{app:privacy_proc}

Algorithm~\ref{alg:dp} lists the private condensation of
Proposition~\ref{prop:dp_rescene} as a client runs it in one round. Each step
embeds the images with a freshly initialized network, as in distribution
matching~\citep{Zhao_2023_WACV}, so the surrogates of all rounds and tasks are
matched under the same distribution of networks. In our
experiments, each client uses $\mathrm{IPC}{=}10$, $J{=}250$ steps, $q{=}0.1$,
$L{=}50$, $C{=}1$, step size $\eta{=}1$, and momentum $\gamma{=}0.5$, and we
consider the noise multipliers $\sigma\in\{1,3,5\}$ of FedDM.

\begin{algorithm}[h]
\small
\SetAlCapFnt{\small}
\SetAlCapNameFnt{\normalfont\small}
\SetAlgoLined
\DontPrintSemicolon
\SetKwComment{Comment}{$\triangleright$\ }{}
\SetCommentSty{textnormal}
\KwIn{local data $\D_t^k$, images per class
$\mathrm{IPC}$, steps $J$, sampling probability $q$, normalization constant $L$,
clip norm $C$, noise multiplier $\sigma$, step size $\eta$, momentum $\gamma$}
\ForEach{class $c\in\Cset_t$}{
  draw the $\mathrm{IPC}$ surrogates $S_{t,c}^{k,(r)}$ from $\mathcal{N}(0,\mathbf{I})$\;
}
\For{$j = 1$ \KwTo $J$}{
  initialize the network $\theta'$ at random\;
  \ForEach{class $c\in\Cset_t$}{
    include each image of $\D_{t,c}^k$ in the lot $\mathcal{Q}$ independently with
    probability $q$\;
    \ForEach{image $x_i\in\mathcal{Q}$}{
      $\bar g_i \leftarrow \tilde g(x_i)\,/\max\!\big(1,\,\lVert\tilde g(x_i)\rVert_2/C\big)$\;
    }
    $\widehat g \leftarrow \tfrac{1}{L}\big(\sum_{i\in\mathcal{Q}}\bar g_i
      +\mathcal{N}(0,\sigma^2C^2\mathbf{I})\big)$
      \Comment*[r]{Eq.~\ref{eq:dp_step}}
    update $S_{t,c}^{k,(r)}$ by an SGD step along $\widehat g$ with step size
    $\eta$ and momentum $\gamma$\;
  }
}
\Return{$S_t^{k,(r)}=\bigcup_c S_{t,c}^{k,(r)}$}\;
\caption{Private condensation at client $k$ in round $r$ of task $t$}
\label{alg:dp}
\end{algorithm}

\noindent\textbf{Privacy budget.}\quad
We compute $\epsilon$ at $\delta{=}10^{-5}$ by R\'enyi differential privacy
accounting of the subsampled Gaussian mechanism~\citep{mironov2019renyidifferentialprivacysampled} with the
tightened conversion to
$(\epsilon,\delta)$-differential privacy~\citep{pmlr-v108-balle20a}. The noise multipliers
$\sigma{=}1$, $3$, and $5$ give $\epsilon{=}12.36$, $2.47$, and $1.36$ per round.
We set $q$ and $J$ so that these noise multipliers give per-round budgets close to
the $12.25$, $2.46$, and $1.35$ of FedDM. A task of $R$ rounds spends
$\epsilon(q,\sigma,RJ;\delta)$ by Sec.~\ref{app:privacy_formal}, which gives
$\epsilon{=}31.20$, $6.00$, and $3.26$ for the $R{=}5$ rounds per task of
Sec.~\ref{app:privacy_emp}. Since $q$,
$\sigma$, and $J$ are public constants, every client obtains the same
$\epsilon$ regardless of how many images it holds.

\subsection{Accuracy Comparison under the Same Privacy Budget}
\label{app:privacy_emp}

We compare the accuracy of ReSCENE with the baselines under the same privacy
budget, as FedDM does in its evaluation. We use MNIST~\citep{726791}
and CIFAR-10 with $\beta{=}0.5$. There are ten clients as in FedDM, and all of
them participate in every round. Following FedDM, we use ConvNet, which has three
convolutional blocks with instance normalization, since the batch normalization of
ResNet-18 is incompatible with DP-SGD. Each task has $R{=}5$ rounds, so the three
noise levels of Sec.~\ref{app:privacy_proc} give the per-task budgets
$\epsilon{=}31.20$, $6.00$, and $3.26$. ReSCENE runs Algorithm~\ref{alg:dp}. The
baselines apply DP-SGD to their local training with clip norm $1$, and we set
their noise multiplier so that the local training of a task spends the same
$\epsilon$ as ReSCENE. The budget of FedCBDR also covers the features that it
uploads, which we release through the Gaussian mechanism. All other settings
follow Sec.~\ref{sec:setup}. Table~\ref{tab:privacy_rescene} reports AA and AIA.

\begin{table}[h]
\centering

\small
\renewcommand{\arraystretch}{1.12}
\setlength{\tabcolsep}{7pt}
\begin{tabular}{llcccccc}
\toprule
& & \multicolumn{2}{c}{$\epsilon{=}31.20$} & \multicolumn{2}{c}{$\epsilon{=}6.00$} & \multicolumn{2}{c}{$\epsilon{=}3.26$}\\
\cmidrule(lr){3-4}\cmidrule(lr){5-6}\cmidrule(lr){7-8}
Dataset & Method & AA & AIA & AA & AIA & AA & AIA\\
\midrule
\multirow{4}{*}{MNIST}
& FedAvg & 18.89\std{0.16} & 44.65\std{0.45} & 18.73\std{0.26} & 43.20\std{1.07} & 18.48\std{0.50} & 41.34\std{2.56}\\
& Re-Fed & 34.22\std{6.37} & 63.10\std{5.12} & 33.07\std{5.20} & 61.31\std{6.48} & 31.86\std{4.67} & 58.33\std{6.03}\\
& FedCBDR & 83.75\std{4.70} & 91.73\std{2.38} & 76.09\std{9.41} & 87.31\std{8.83} & 71.54\std{12.10} & 81.03\std{15.60}\\
& \cellcolor{oursbg}ReSCENE & \cellcolor{oursbg}\textbf{86.94}\std{0.36} & \cellcolor{oursbg}\textbf{93.33}\std{0.29} & \cellcolor{oursbg}\textbf{83.03}\std{0.47} & \cellcolor{oursbg}\textbf{89.97}\std{0.33} & \cellcolor{oursbg}\textbf{80.34}\std{0.24} & \cellcolor{oursbg}\textbf{88.80}\std{0.77}\\
\midrule
\multirow{4}{*}{CIFAR-10}
& FedAvg & 13.49\std{0.08} & 34.09\std{1.46} & 13.09\std{0.20} & 30.75\std{2.76} & 12.63\std{0.61} & 28.85\std{3.24}\\
& Re-Fed & 23.17\std{0.95} & 42.38\std{1.26} & 21.31\std{2.12} & 40.17\std{2.84} & 18.76\std{1.18} & 37.87\std{2.68}\\
& FedCBDR & 30.36\std{1.86} & 49.25\std{1.83} & 28.63\std{2.92} & 46.70\std{3.23} & 24.10\std{2.75} & 42.48\std{3.19}\\
& \cellcolor{oursbg}ReSCENE & \cellcolor{oursbg}\textbf{47.62}\std{1.12} & \cellcolor{oursbg}\textbf{64.24}\std{0.93} & \cellcolor{oursbg}\textbf{44.85}\std{0.86} & \cellcolor{oursbg}\textbf{61.42}\std{0.42} & \cellcolor{oursbg}\textbf{43.95}\std{0.88} & \cellcolor{oursbg}\textbf{58.71}\std{0.92}\\
\bottomrule
\end{tabular}
\caption{AA and AIA (\%) on MNIST and CIFAR-10 ($\beta{=}0.5$) under differential
privacy. Each pair of columns fixes the per-task privacy budget $\epsilon$ at
$\delta{=}10^{-5}$ for all methods. We report means and standard deviations over
three seeds, and the best mean in each column is in bold.}
\label{tab:privacy_rescene}
\end{table}

\noindent\textbf{Results.}\quad
ReSCENE obtains the highest mean AA and AIA at every budget on both datasets. Its
AA is $3.19$ to $8.80$ points above that of FedCBDR, the strongest baseline, on
MNIST and $16.22$ to $19.85$ points above it on CIFAR-10. ReSCENE spends the
budget when a client condenses its data, and the server then trains the global
model on the uploaded surrogates without additional noise, whereas the baselines
add noise to every step of their local training.

\subsection{Visualization of the Uploaded Surrogates}
\label{app:privacy_vis}

Figure~\ref{fig:dp_surrogates} visualizes the surrogates that one client uploads
with Algorithm~\ref{alg:dp} in the setting of Sec.~\ref{app:privacy_emp}. The
surrogates start from Gaussian noise, and we show them without gradient noise
($\sigma{=}0$) for reference and at the three noise levels of
Table~\ref{tab:privacy_rescene}.

\begin{figure}[!h]
\centering
\includegraphics[width=0.95\linewidth]{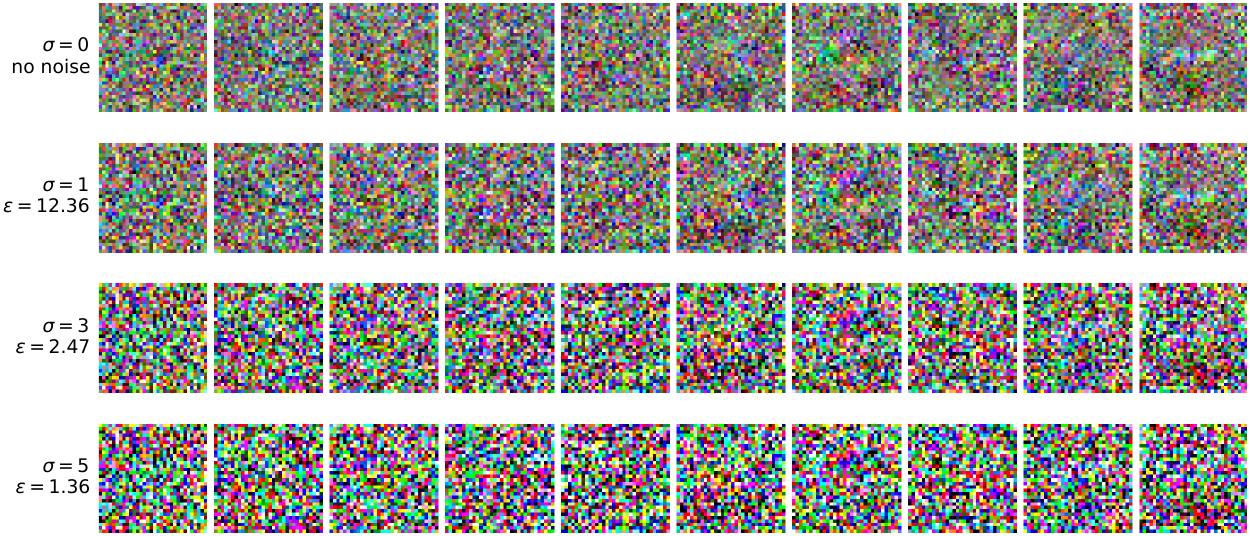}
\caption{Surrogates that one client uploads with Algorithm~\ref{alg:dp} on
CIFAR-10. Each row shows five of the ten surrogates of each of the two classes of
the first task at its last round. The top row adds no gradient noise
($\sigma{=}0$) and is shown for reference, and each of the other rows uses one
noise multiplier $\sigma$ with its per-round $\epsilon$. Pixel values are mapped back to the image
range and clipped for display.}
\label{fig:dp_surrogates}
\end{figure}

\end{document}